%% file: main.tex
\documentclass[sigconf, nonacm]{acmart}
\AtBeginDocument{%
  \providecommand\BibTeX{{%
    \normalfont B\kern-0.5em{\scshape i\kern-0.25em b}\kern-0.8em\TeX}}}

\usepackage{comment}
\usepackage{amsmath}
\usepackage{mathtools}
\usepackage{nicefrac}
\usepackage{bigints}
\usepackage{soul}
\usepackage{bigdelim}
\usepackage{arydshln}
\usepackage{picture}

\newcommand{\mob}{M\"{o}bius}

\newcommand{\RS}{{\ensuremath{\widehat{\mathbb{C}}}}}

\newcommand{\R}{{\ensuremath{\mathbb R}}}

\newcommand{\SLC}{{\ensuremath{\textrm{SL}(2, \mathbb{C}})}}
\newcommand{\SUTWO}{{\ensuremath{\textrm{SU}(2)}}}
\newcommand{\OPS}{{\ensuremath{\textrm{L}}}} % origin-preserving subgroup

\newcommand{\lnxf}[2]{{\ensuremath{D_{#2}^{#1}}}}

\newcommand{\scalef}[2]{\ensuremath{\lambda_{#1}^2(#2)}}

\newcommand{\T}{\ensuremath{{\mathfrak{ T}}}}

\newcommand{\h}{\ensuremath{\rho}}

\newcommand{\LB}{\ensuremath{{ {\boldsymbol{\zeta}}}}}
\newcommand{\BB}{\ensuremath{{ {\boldsymbol{\xi}}}}}
\newcommand{\fB}{\ensuremath{{ {\mathbf B}}}}

\newcommand{\feat}{\ensuremath{\psi}}

\newtheorem{claim}{Claim}

\newcommand{\Fig}[1]{Figure~\ref{#1}}

\DeclarePairedDelimiterX{\abs}[1]{\lvert}{\rvert}{#1}
\newcommand{\eval}[2]{\ensuremath{\left. #1 \right\rvert_{#2}}}

\begin{document}

\title{\mob{} Convolutions for Spherical CNNs}

\author{Thomas W. Mitchel}
\email{tmitchel@jhu.edu}
\affiliation{%
  \institution{Johns Hopkins University}
}
\author{Noam Aigerman}
\email{aigerman@adobe.com}
\author{Vladimir G. Kim}
\email{vokim@adobe.com}
\affiliation{
  \institution{Adobe Research}
}

\author{Michael Kazhdan}
\email{misha@cs.jhu.edu}
\affiliation{%
  \institution{Johns Hopkins University}
}

\renewcommand{\shortauthors}{Mitchel, et al.}

%%%%%%%%%%%%%%%%%%%%%
%%%% ABSTRACT %%%%%%%
%%%%%%%%%%%%%%%%%%%%%

\begin{abstract}
\input{abstract}
\end{abstract}

%%%%%%%%%%%%%%%%%%%%%%%%%%
%%%% CSS + KEYWORDS %%%%%%
%%%%%%%%%%%%%%%%%%%%%%%%%%




\begin{teaserfigure}
\centering
\begin{picture}(0.98\linewidth,0.36\columnwidth)
\put(0, 0){\includegraphics[width=0.97\linewidth]{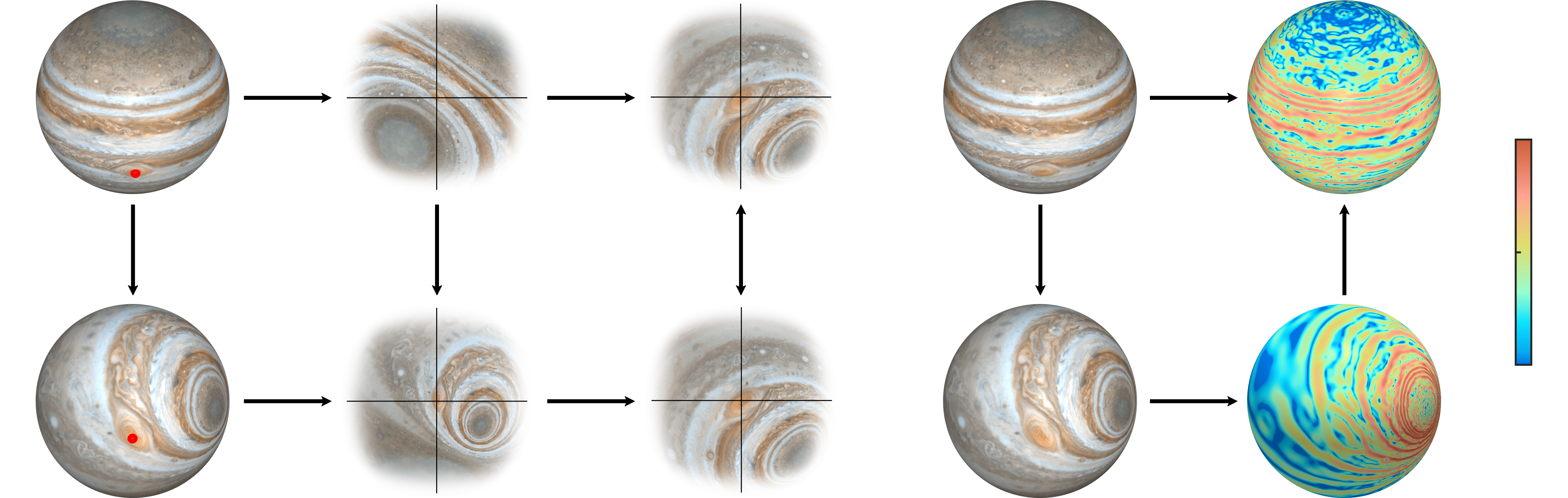}}
\put(108, 168){Frame Operator}
\put(340.5, 168){Density Operator}
\put(-1, 122){$\psi$}
\put(-3, 26,5){$g\psi$}
\put(31, 76){$g$}
\put(120, 76){$\lnxf{g}{z}$}
\put(222, 76){$e$}
\put(82, 134.5){$\log_{z}$}
\put(80.5, 39.5){$\log_{gz}$}
\put(173, 134.5){$\T_{\psi}(z)$}
\put(170, 39.5){$\T_{g\psi}(gz)$}
\put(39, 107){$z$}
\put(33.5, 23.5){$gz$}
\put(316, 77){$g$}
\put(405, 76){$\lambda_{g}^2\boldsymbol{\cdot}$}
\put(368, 134.5){$\h_{\psi}$}
\put(366.5, 39){$\h_{g\psi}$}
\put(466.5, 74){$0$}
\put(466.5, 109.25){$+$}
\put(466.5, 39.5){$-$}
\put(484, 74){$\textrm{ln} \, \h$}
\end{picture}
\caption{ Equivariant convolutions require two ingredients: a \textit{frame operator} and a \textit{density operator}. Filters assign weights based on the relative positions of points and the frame operator $\T$ corrects for the deformation of the local ``tangent space'' under a \mob{} transformation $g$. Similarly, the density operator $\rho$ adjusts for the change in the area measure used for integration, proportional to the conformal scale factor $\lambda_{g}^2$.\label{fig:teaser}}
\end{teaserfigure}
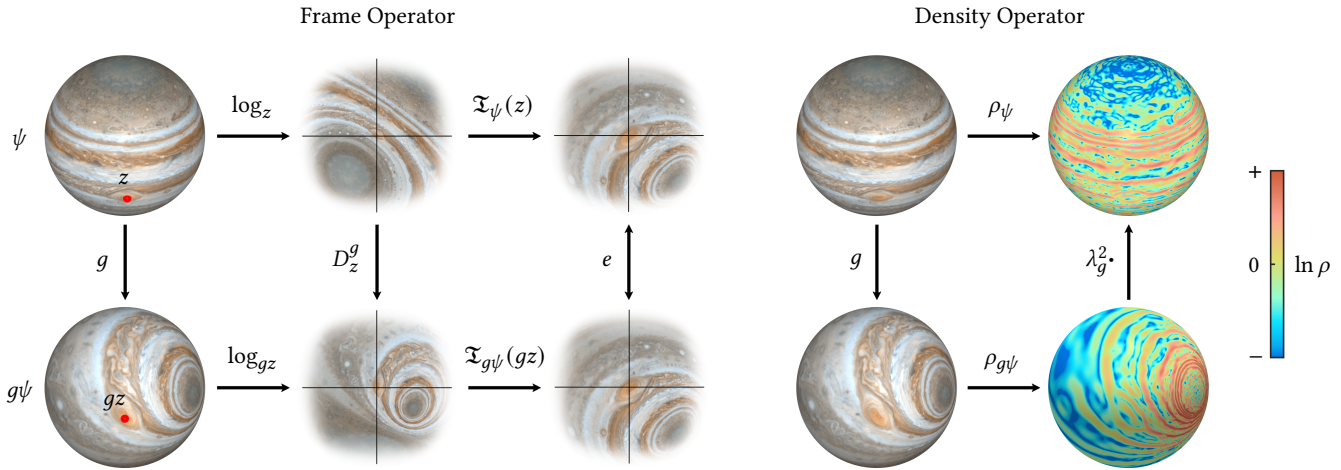

\maketitle

%%% Intro
\section{Introduction}
\label{s:intro}
\input{intro}

\section{Related Work}
\label{s:related}
\input{related}

\section{Method Overview}
\label{s:overview}
\input{overview}

\section{M\"obius Transformations}
\label{s:mobius}
\input{mobius}

\section{M\"obius Convolution}
\label{s:conv}
\input{conv}

\section{Discretization}
\label{s:discretization}
\input{discretization}

\section{\mob{}-Equivariant Spherical CNNS}
\label{s:network}
\input{network}

\section{Evaluation}
\label{s:evaluation}
\input{evaluation}

\section{Conclusion}
\label{s:conclusion}
\input{conclusion}

\bibliographystyle{ACM-Reference-Format}
\bibliography{main}

\appendix

\section{Integration and Differentiation}
\label{a:intdiff}
\input{supp/supp_intdiff}

\section{Equivariance of \mob{} Convolutions}
\label{a:equivariance}
\input{supp/supp_equiv}

\section{Transformation of Functions}
\label{a:filter_xform}
\input{supp/supp_xform}

\section{Memory Footprint and Run Time}
\label{a:overhead}
\input{supp/supp_overhead}

\section{Segmentation}
\label{a:seg}
\input{supp/supp_seg}

\end{document}

%% file: abstract.tex
 \mob{} transformations play an important role in both geometry and spherical image processing -- they are the group of conformal automorphisms of 2D surfaces and the spherical equivalent of homographies. Here we present a novel, \mob{}-equivariant spherical convolution operator which we call \mob{} convolution; with it, we develop the foundations for \mob{}-equivariant spherical CNNs.  Our approach is based on the following observation: to achieve equivariance, we only need to consider the lower-dimensional subgroup which transforms the positions of points as seen in the frames of their neighbors.   To efficiently compute \mob{} convolutions at scale we derive an approximation of the action of the transformations on spherical filters, allowing us to compute our convolutions in the spectral domain with the fast Spherical Harmonic Transform. The resulting framework is flexible and descriptive, and we demonstrate its utility by achieving promising results in both shape classification and image segmentation tasks. 

%% file: intro.tex
 Convolutional neural networks (\textbf{CNNs}) are effective because convolution responds to a contextualized window, forcing the learning to be \textit{translation-equivariant}. However, vanilla CNNs assume a fixed coordinate frame and lose effectiveness in the presence of deformations that change the frame. This has lead to the development of more general notions of convolution equivariant to transformation groups including rotations \cite{cohen2016group, worrall2017harmonic} and dilations \cite{worrall2019deep, sosnovik2019scale, finzi2020generalizing}. Critically, the notion of rotation-equivariance has facilitated the generalization of CNNs to domains without a canonical orientation at each point such as the sphere \cite{cohen2018spherical, Cohen2019, esteves2020spin} and arbitrary surfaces \cite{de2020gauge, wiersma2020cnns, Mitchel_2021_ICCV}. The resulting networks are \textit{isometry-equivariant} -- able to repeatably characterize local features in the presence of distance-preserving transformations -- and have excelled in fundamental geometry processing tasks such as shape classification, segmentation, and correspondence. 
 
 Despite their success, rotation- and isometry-equivariant CNNs can fail to achieve adequate performance in the presence of the kinds of complex deformations commonly found in real-world image and shape data \cite{Mitchel_2021_ICCV}. Such deformations may potentially be better modeled by higher-dimensional transformation groups. For example, homographies (projective transformations) better approximate changes in camera viewpoints than similarities (rotations and dilations) \cite{hartley2003multiple} and, for spherical images, can be represented using conformal transformations~\cite{el2014conformal,SchleimerS16squares}. For geometry processing, conformal (angle-preserving) transformations encompass a broader class of deformations than isometries that still preserve the sense of ``shape". \cite{Levy:TOG:2002,Gu:TMI:2004,Crane:2011:STD}. 
 
 We present a novel spherical convolution operator, equivariant to \mob{} transformations, which we call \textit{\mob{} convolution} (\textbf{MC}). Using this operator, we develop the foundations for \textit{\mob{}-equivariant} spherical CNNs. Our convolutions are flexible, and we demonstrate the utility of our \mob{}-equivariant CNNs by achieving promising results on standard benchmarks in both genus-zero shape classification and spherical image segmentation.
 
 Our approach is based on a key observation: while the \mob{} group is six-dimensional, its action on the characterization of the position of a point relative to its neighbor can be described by a four-dimensional subgroup. By defining an equivariant frame operator (\Fig{fig:teaser}, left) at each point with which we align the filter, we correct for the change in the relative positions induced by this subgroup. To compute convolutions, input features are mapped to a density distribution (\Fig{fig:teaser}, right), controlling for the change in area measure, and integrated against the aligned filters over the sphere, rather than the group itself.  To facilitate efficient evaluations, we parameterize filters using log-polar basis functions from which we derive an approximation of the action of the frames, allowing us to compute our convolutions via the Fast Spherical Harmonic Transform \cite{driscoll1994computing, kostelec2008ffts}. Our implementation is publicly available at \href{https://github.com/twmitchel/MobiusConv}{\texttt{github.com/twmitchel/MobiusConv}}.

%% file: related.tex
Group-equivariant CNNs were first introduced by \cite{cohen2016group, cohen2019general}, wherein kernels are parameterized in terms of equivariant basis functions \textit{on the group itself}. Convolution is performed by lifting features from the domain and searching over all possible transformations of the kernels. This becomes intractable for non-compact groups where the domains of integration are unbounded and the representations are infinite-dimensional, though recent work by \cite{finzi2020generalizing} mitigates these problems by considering only the origin-preserving subgroups and integrating with respect to an equivariant Monte Carlo estimator to facilitate evaluations. This approach is flexible, and has since been extended to handle general affine transformations and homographies \cite{macdonaldLie}. In tandem, several \textit{fully-connected} networks have been developed that achieve equivariance to non-compact transformation groups including the Lorentz group \cite{bogatskiy2020lorentz}, the Poincar\'{e} group \cite{villar2021scalars}, and the symplectic group \cite{finzi2021practical} -- all closely related to \mob{} transformations.

Equivariant CNNs that integrate over the domain on which the group acts generally parameterize kernels in terms of basis functions that rotate or dilate with the local coordinate system \cite{worrall2017harmonic, weiler2019general, worrall2019deep, sosnovik2019scale}. This approach has been extended to volumetric domains \cite{weiler2019general},  point clouds \cite{qi2017pointnet}, and canonical domains such as the sphere \cite{cohen2018spherical, esteves2020spin}. However, finite-dimensional equivariant bases often don't exist for non-commutative and non-compact transformation groups of interest, precluding generalizations to groups that act projectively.

Equivariance is a necessary condition for the transposition of convolutional frameworks to domains without canonical coordinate systems such as arbitrary 2D surfaces. Existing surface networks can be broadly categorized into several emerging paradigms: \textit{Representational} methods \cite{hanocka2019meshcnn, lahav2020meshwalker} exploit the ubiquitous representation of surfaces as triangle meshes to form operators equivariant to local similarities; \textit{Diffusive} approaches \cite{sharp2020diffusion, smirnov2021hodgenet, Yi_2017_CVPR} formulate convolution in terms of spectral kernels and accelerate computations in a low-frequency eigenbasis; and \textit{Transporting} networks \cite{Mitchel_2021_ICCV, wiersma2020cnns} propagate tangent vector features that transform with local coordinate systems. Despite their success in a variety of tasks, existing surface networks are only repeatable up to isometries, an we believe our approach to be the first surface network equivariant to conformal transformations.

%% file: overview.tex
Our method is based on a powerful observation: instead of dealing with the six-dimensional group of \mob{} transformations, we only need to consider a four-dimensional subgroup. We first review the action of \mob{} transformations on the sphere, and define a notion of relative positions between points analogous to the logarithm map on surfaces, and show how the latter transforms under the former. 

We then introduce \mob{} convolutions, which provide a method for \mob{}-equivariant spatial aggregations on the sphere.  \mob{} convolutions are part of the \textit{extended convolution} framework \cite{mitchel2020efficient}  which allows filters to adaptively transform as they shift over a manifold by defining an equivariant frame at each point. This framework forms the basis for recently proposed state-of-the art surface descriptors \cite{mitchel2020echo} and CNNs \cite{Mitchel_2021_ICCV}. We facilitate an efficient discretization by parameterizing filters using log-polar basis functions from which we derive a linearized approximation of the action of the frames, allowing us to compute our convolutions via the fast spherical harmonic transform \cite{driscoll1994computing, kostelec2008ffts}.

We complete the foundations for \mob{}-equivariant CNNs by introducing a conformally-equivariant normalization layer based on filter response normalization \cite{singh2020filter} and we validate equivariance by direct experimental evaluation.  The principle module in applications is a simple \mob{} convolution ResNet (\textbf{MCResNet}) block \cite{he2016deep}, which is self-contained and flexible.  We demonstrate the  utility of our framework by achieving promising results on standard benchmarks in both genus-zero shape classification and spherical image segmentation.

%% file: mobius.tex
\mob{} transformations can be understood by associating the two-sphere with the \textit{Riemann sphere},  $\RS = \mathbb{C} \cup \{\infty\}$, via the  stereographic projection taking the north pole to the origin.
%, realized in spherical coordinates as
%\begin{align}
%    z(\theta, \, \phi) \equiv \tan \frac{\phi}{2} \, e^{-i \theta}. \in \RS\qquad\phi\in[0,\pi],\ \theta\in[0,2\pi) \label{s_proj}
%\end{align}
\mob{} transformations are described by the action of \SLC{}, the group of matrices in $\mathbb{C}^{2 \times 2}$ with unit determinant, on \RS{} by fractional linear transformations. That is, for any $g = \left[ \begin{smallmatrix} a & b \\ c & d \end{smallmatrix}\right] \in \SLC{}$ and $z \in \RS$, 
\begin{align}
    g z  \equiv \frac{a z + b}{c z + d}. \label{flt}
\end{align}

\subsection{Generalizing the logarithm and exponential maps}
To parameterize the sphere about a point $z\in\RS$, we borrow the notion of the exponential and logarithm maps from Riemannian geometry, defining the \textit{generalized logarithm} of $z$ as a rotation $\log_z \in \SUTWO{}$ taking $z$ to the origin,
\begin{align}
\log_z \equiv \frac{1}{|c|\sqrt{1+|z|^2}}\left(\begin{array}{cc}c & -cz \\ \bar{c}\bar{z} & \bar{c}\end{array}\right) \label{gen_log}
\end{align}
with $c=\sqrt{\bar{z}}$. (Though any choice of $c$ gives a rotation taking $z$ to the origin, the above choice ensures that the great circle going through the origin and $z$ is mapped to the real line, enabling the use of the fast Spherical Harmonic Transform in the implementation of \S\ref{s:discretization}.)

Then, for any point $y \in \RS$, we can express the ``position" of $y$ in the frame of $z$ as $\log_z  y \in \RS$. By analogy to Riemmannian geometry, the generalized logarithm maps $\RS$ to the ``tangent space" at $z$. We make this explicit by denoting the image of the logarithm map at $z$ as $\RS_z$:
$$\log_z:\RS\rightarrow\RS_z$$
though formally $\RS$ and $\RS_z$ are the same space -- the Riemann sphere. Similarly, we define the \textit{generalized exponential} of $z$ as the inverse of the generalized logarithm, $\exp_z \equiv \log_z^{-1}:\RS_z\rightarrow\RS$.

\subsection{Action of the origin-preserving subgroup}
We will think of convolution filters as functions defined on a canonical ``tangent space'' describing the weight with which a point contributes to its neighbor in terms of the position of the neighbor in the frame of that point. To design our network, we need to define a \mob{}-equivariant convolution operator. To this end we need to understand how the position of a point in the frame of its neighbor changes under the action of \mob{} transformations.

Describing the transformation from one coordinate frame to the other is straight-forward: Beginning in $\RS_z$, we 1). Map to $\RS$ by applying $\exp_z$; 2). Transform $\RS$ by $g$; and 3). Map back to $\RS_{gz}$ using $\log_{gz}$. Composing these give the \mob{} transformation,
\begin{align}
    \lnxf{g}{z} \equiv \log_{gz}\circ \ g\circ \exp_z:\RS_z\rightarrow\RS_{gz} \in \SLC{} , \label{param_xform}
\end{align}
with the notation chosen to reflect dependence on both $z$ and $g$. 

From Equation~(\ref{param_xform}) and the definitions of the generalized logarithm and exponential, $\lnxf{g}{z}$ must belong to the origin-preserving subgroup $\OPS{} \subset \SLC{}$, consisting of the lower-triangular elements of \SLC{}. This follows from the facts that $z$ maps to $gz$ under the action of $g$ and that both $z$ and $gz$ are the origin in their respective tangent spaces. This simple but critical observation implies that in defining equivariant convolution, we only need to consider the four-dimensional origin-preserving subgroup, not the full, six-dimensional group of \mob{} transformations. 

%% file: conv.tex
As in \cite{mitchel2020efficient} we implement convolution by shifting a filter over the domain, aligning the shifted filter using a frame field, and distributing the values of a density function to neighboring points, with distribution weights given by the aligned filter.

Rather than have the user provide the transformation field and density directly, our framework derives these from an input signal in such a way so as to ensure the ensuant convolution is \mob{}-equivariant. Specifically, we construct a \textit{frame operator} $\T$ and \textit{density operator} $\h$ that take in a real-valued function on $\RS$ and return a lower-triangular frame field and a real-valued density field, respectively,
\begin{align}
\begin{aligned}
\T:L^2(\RS,\R) &\rightarrow L^2(\RS, \, \textrm{L}) \\
\feat{} & \mapsto \T_{\feat{}}
\end{aligned}
\quad \textrm{and} \quad 
\begin{aligned}
\h: L^2(\RS,\R) &\rightarrow L^2(\RS,\R) \\
\feat{} & \mapsto \h_{\feat{}}
\end{aligned} 
\label{trans_maps}
\end{align}
Given a \mob{} transformation $g\in\SLC$ and a point on the Riemann sphere, $z\in\RS$, the frame operator corrects for the deformation of the tangent space resulting from $g$, as characterized by the origin-preserving transformation $D_z^g$ from Equation~(\ref{param_xform}) (Figure~\ref{fig:teaser}, left). Similarly, the density operator adjusts for the change in the area measure used for integration, given by the conformal scale factor $\scalef{g}{z}$ (Figure~\ref{fig:teaser}, right).

Given operators $\T$ and $\h$, the \mob{} convolution of a function $\feat{}$ with with a filter $f$, both in $L^2(\RS,\R)$, can formally be expressed as the function in $L^2(\RS,\R)$ with
\begin{align}
    (\feat * f)(y) &= \int_{\RS} \h_{\feat{}}(z) \,
    \big[\T_{\feat{}}(z)  \, f \big]\big( \log_z y \big) \ dz \label{m_conv},
\end{align}
where $\T_{\feat{}}(z)  \, f \equiv f \circ \big[\T_{\feat{}}(z)\big]^{-1}$ denotes the standard action of \mob{} transformations on $f$ by left shifts.

That is, to get the value at a point $y\in\RS$, we 1).~Iterate over all neighbors $z$; 2).~Compute the position of $y$ in the frame at $z$; 3).~Evaluate the filter at that point; and 4).~Accumulate the density at $z$ weighted by the filter value.
Following~\cite{mitchel2020efficient}, instead of aggregating features with respect to a single frame at each point, the influence of each frame is spread across the neighbors, making the construction robust to noise and other nuisance factors affecting the stability of the frame field.

\subsection{Equivariance}
Following the above discussion, $\T$ and $\h$ must satisfy certain conditions to ensure that \mob{} convolutions are \mob{}-equivariant; \textit{i.e.} that for any function $\feat{}$, filter $f$, and \mob{} transformation $g$, \mob{} convolution commutes with the action of $g$ by left shifts,
\begin{align}
   g \, (\feat{} * f) = (g \, \feat{} * f). \label{equiv_conv}
\end{align}
A sufficient condition for the equation to hold is if for all $\feat{} \in L^2(\RS,\R)$ and $g \in \SLC{}$, the operators $\T$ and $\h$ satisfy
\begin{align}
\lnxf{g}{z} \, \T_{\feat{}}(z) = \T_{g \, \feat{}}(gz)
\qquad \textrm{and} \qquad 
\lambda_g^{-2}(z) \, \h_{\feat{}}(z) = \h_{g \, \feat{}}(gz),
\label{hT_cond}
\end{align} 
for all $z \in \RS$. The condition for $\T$ follows from Equation~(\ref{param_xform}), and that on $\h$ comes from the change of variables in the integral of Equation~(\ref{m_conv}). 
In particular these conditions are satisfied when the frame and density operators are defined as
\begin{align}
    \T_{\psi}(x) &\equiv \begin{bmatrix} \left( d \log_{x}\feat{} \big\vert_{0}\right)^{-\frac{1}{2}}  & 0 \\
    \left(\frac{1}{2} \nabla d  \log_{x} \feat{} \big\vert_{0}\right) \left( d \log_{x}\feat{}  \big\vert_{0}\right)^{-\frac{3}{2}}& \left( d \log_{x}\feat{} \big\vert_{0}\right)^{\frac{1}{2}}\end{bmatrix} \label{deriv_frame}
\end{align}
and
\begin{align}
    \h_{\psi}(x) &\equiv \abs[\Big]{ \, d \log_{x}\feat{} \big\vert_{0} \, }^{\, 2}, \label{deriv_density}
\end{align}
where $d \log_{x}\feat{}\big\vert_{0}, \, \nabla d \, \log_{x} \feat{}\big\vert_{0} \in \mathbb{C}$ are the differential and Hessian of $\log_{x} \feat{}
$ evaluated at the origin. We note that at a point where $\T_{\feat{}}(x)$ is ill-defined -- those at which $d \log_x \feat{} \big\vert_{0}$ vanishes -- $\h_{\feat{}}(x)$ also vanishes and the point contributes nothing to the convolution.

%We probably don't want to get into this, but the whole convolution vs. correlation discussion is actually even better. In particular, as long as the density is bounded and the set of points on which frame-field is ill-defined and the density function is non-zero has measure zero, the convolution is defined everywhere. This is because the contribution of ``bad'' points to any value in the convolution has measure zero. In contrast, with correlation the convolution will be undefined at ``bad points'' because they contribute to themselves with full measure.

Proof that the conditions in Equation~(\ref{hT_cond}) imply \mob{}-equivariance, and that the operators defined in Equations~(\ref{deriv_frame}--\ref{deriv_density}) satisfy the conditions of Equation~(\ref{hT_cond}) is provided in the supplement (\S\ref{a:equivariance}).

%% file: discretization.tex
To efficiently compute \mob{} convolutions at the scale necessary to build CNNs, we develop an implementation based on the fast Spherical Harmonic Transform \cite{driscoll1994computing, kostelec2008ffts}.  We give an outline of this process below and leave the details to the supplement (\S\ref{a:filter_xform}). 

\subsubsection*{Identity Convolution with the Spherical Harmonic Transform}
To simplify the calculation, we first consider a simpler non-equivariant convolution, we call an \textit{identity convolution}, where we replace the frame and density operators from Equation~(\ref{m_conv}) with the trivial frame field $\T_{\psi}(z)=e$ (with $e$ the identity) and the density $\h_\psi(z)=\psi(z)$, 
\begin{align}
    (\feat{} \, {*}_{e} \, f)(y) &= \int_{\RS} \feat{}(z) \,  f\big(\log_z y \big) \ dz \label{e_conv},
\end{align}
using ${*}_{e}$ to distinguish it from the equivariant convolution.

Assuming that $\feat{}$ and $f$ are $B$ band-limited functions, they can be expressed in terms of their spherical harmonic decompositions as 
\begin{align}
    \feat{} = \sum_{l=0}^{B-1} \sum_{m=-l}^l \boldsymbol{\feat{}}_{lm} \, Y_l^m  \quad \textrm{and} \quad f = \sum_{l'=0}^{B-1} \sum_{m'=-l'}^{l'} {\mathbf f}_{l'm'} \, Y_{l'}^{m'}.\label{fn_decomp}
\end{align}
Recalling that $\log_z$ is a rotation, we know that it preserves the frequency content and expand
\begin{align*}
f\circ\log_z
= \sum_{l'=0}^{B-1}\sum_{|m'|\leq l'}\sum_{|m''|\leq l'}{\mathbf f}_{l'm'}D_{-m'm''}^{l'}\big(\log_z\big)Y_{l'}^{m''}
\end{align*}
where $D_{-m'm''}^{l'}$ is the Wigner-D function giving the $(l',m'')$-th spherical harmonic coefficient of the rotation of $Y_{l'}^{m'}$. Furthermore, using the fact that in the $z-y-z$ Euler angle notation our definition of $\log_z$ corresponds to a rotation described by an Euler triplet whose first entry is zero, it follows that the integral vanishes
\begin{align*}
0 = \int_{\RS}Y_l^m(z)\cdot  D_{-m'm''}^{l'}\big(\log_z\big)\ dz
\end{align*}
whenever $m''\neq m$. Thus, expanding Equation~(\ref{e_conv}) we get
\begin{align}
    (\feat{} \, {*}_{e} \, f)(y) &= \sum_{l'=0}^{B-1}\sum_{m''=-l'}^{l'}
    \left[\sum_{l=|m''|}^{B-1}\boldsymbol{\feat}_{lm''} \left(\sum_{m'=-l'}^{l'} {\mathbf f}_{l'm'} 
    \varDelta_{ll'}^{m'm''}\right)
    \right]Y_{l'}^{m''}
    \label{e_conv_decomp}
\end{align}
where the value of $\varDelta_{ll'}^{m'm''}$ is independent of $\feat{}$ and $f$,
\begin{align}
\varDelta_{ll'}^{m'm''} =\int_{\RS}Y_l^{m''}(z)\cdot D_{-m'm''}^{l'}\big(\log_z\big)\ dz. \label{conv_coeff}
\end{align}

From this, we can compute identity convolutions efficiently by: 1). Computing the spherical harmonic coefficients of $\feat{}$ and $f$ using the fast SHT; 2). Summing the two sets of coefficients according to Equation~(\ref{e_conv_decomp}) to get the coefficients of the convolution; and 3). Applying the fast inverse SHT to reconstruct the convolution.

The complexity of steps 1 and 3 are proportional to those of the fast SHT, which is $O(B^2 \log^2B)$. However, in our implementation we compute the discrete Legendre transform via sparse matrix multiplication for a total complexity of proportional to $O(B^3 \log B)$, which we find to be more efficient on the GPU. Step 2 has complexity $O(B^4)$, and we show that this computation can be re-used in computing the full (equivariant) convolution.

\begin{figure}[t]
\centering
\begin{picture}(0.9\linewidth,0.67\columnwidth)
\put(5,10){{\includegraphics[width=0.9\linewidth]{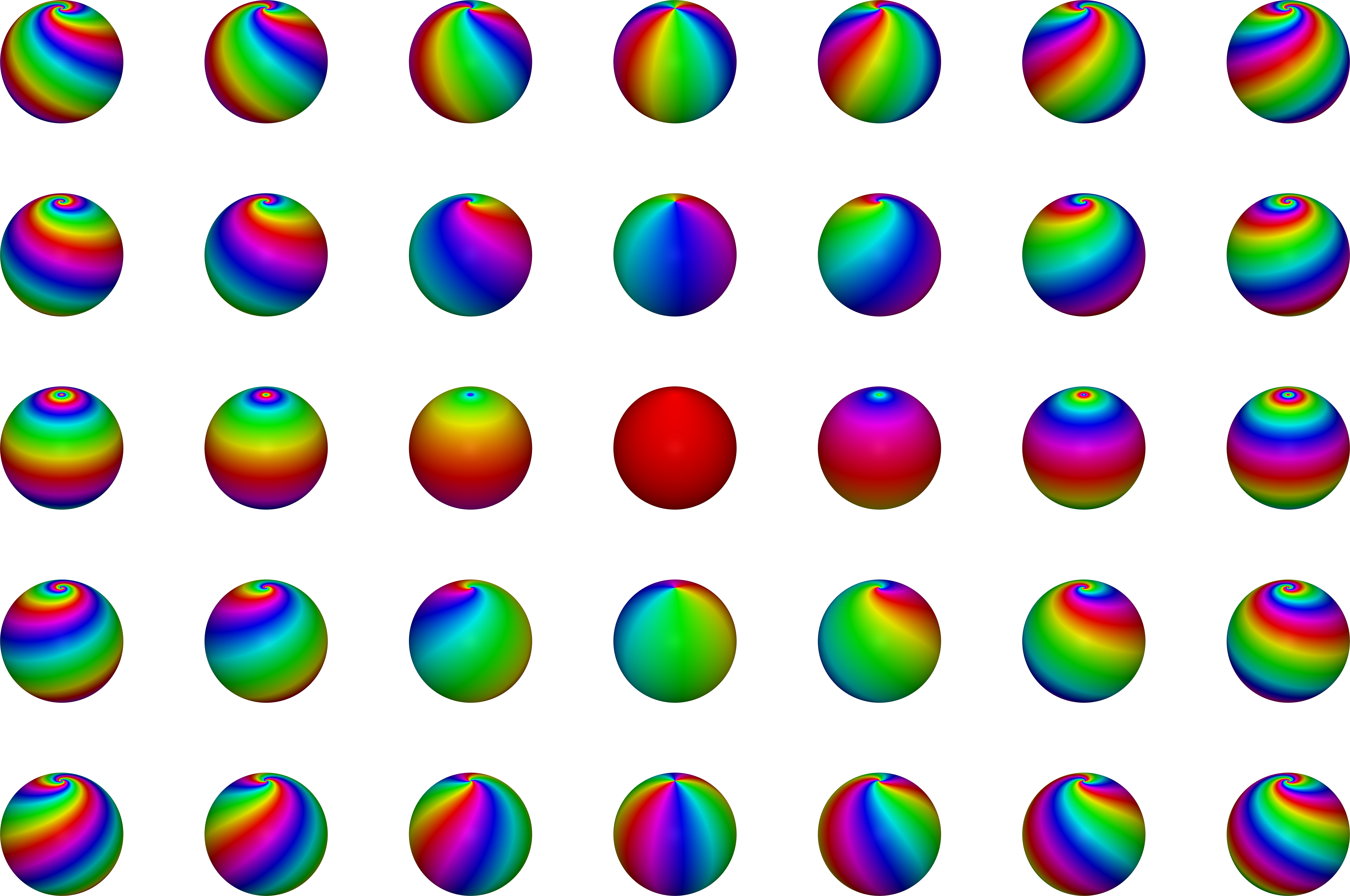}}}
\put(112,0){$s$}
\put(122,2){\vector(1,0){50}}
\put(106,2){\vector(-1,0){50}}
\put(-10,80){$m$}
\put(-7,93){\vector(0,1){30}}
\put(-7,72){\vector(0,-1){30}}
\end{picture}
    \caption{Spherical log-polar functions $\fB_{ms}^t$ at integer exponents with $t=0.15$.
    \label{log_polar_trip}}
\end{figure}
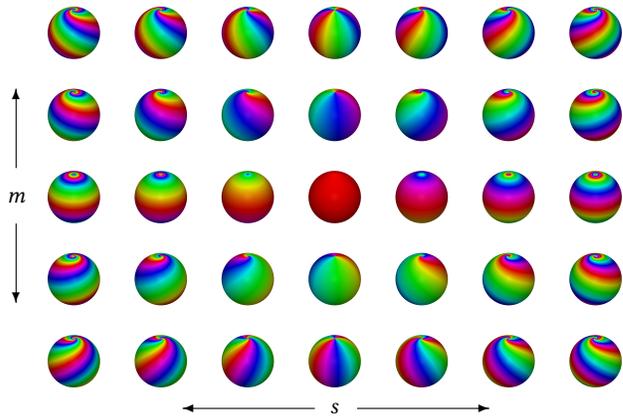

\subsubsection*{Spherical Log-Polar Bases} To efficiently compute arbitrary \mob{} convolutions as in Equation~(\ref{m_conv}), we approximate them as sums of identity convolutions. To do so, we choose a basis for our filters that enables an approximation of the action of $\T_{\psi}$.

To this end we use linear combinations of log-polar (Fourier-Mellin) basis functions \cite{vilenkin1978special,vilenkin1991representation},
\begin{align}
    \fB_{ms}^t(z) \equiv \frac{|z|^{is}}{|z|^t} \left(\frac{z}{|z|}\right)^m  \label{log_polar}
\end{align}
with $m\in{\mathbb Z}$ and $s,t\in{\mathbb R}$. These function are localized about $z=0$ (resp. $z=\infty$) when $t>0$ (resp. $t<0$), are discontinuous at $z\in\{0,\infty\}$ when $m\neq0$ (since the argument of $z$ is not defined), and singular at $z=0$ (resp. $z=\infty$) when $t>0$ (resp. $t<0$). We note that, for the purposes of integration, the singularity can be ignored when $t<1$. Loosely, this follows from the fact that $\int\frac{1}{x^t}=\frac{1}{1-t}x^{1-t}+c$ which is bounded at $x=0$ whenever $t\in(0,1)$. Noting also that $\fB_{ms}^{t}(z) = \fB_{-m-s}^{-t}(1/z)$, it follows that the functions are continuous away from $\{0,\infty\}$ and, for the purposes of integration, singularities at $\{0,\infty\}$ can be ignored when $|t|<1$.

The first several basis functions at integer frequencies $\abs[]{m} \leq 2$ and $\abs[]{s} \leq 3$, with $t = 0.15$ are shown in Figure~{\ref{log_polar_trip}}. The complex-valued functions are visualized using the HSV scale: hue and value are determined by the arguments and magnitudes of the function values, and saturation is fixed at one.

Using these as basis functions, we consider filters in the span \begin{align}
    f = \sum_{m=-M}^M\sum_{s=-N}^N  b_{ms} \, \fB_{ms}^t \label{filter_param}
\end{align}
where $s$ is constrained to be an integer and $t$ is positive so as to localize the filter about the origin. As our filters are real-valued we have $\overline{b_{-m-n}} = b_{mn}$, giving $(2M + 1)(2N + 1)$ real parameters. 

\subsubsection*{Approximating the Transformation of Filters} Unfortunately, it is not the case that the space of band-limited filters spanned by the $\fB_{ms}^t$ is fixed under the action of the lower-triangular subgroup. This is because, in general, for non-compact, non-commutative groups, a space of functions fixed under the action of the group (i.e. a representation) will be infinite-dimensional.

However, we show in the supplement (\S\ref{a:filter_xform}) that, given a filter as in Equation~(\ref{filter_param}) and a lower-triangular matrix $L\in\OPS{}$, the transformation of the filter $f$ by $L$ can be expanded as
\begin{align}
    L \, f = \sum_{m = -\infty}^{\infty} \int_{-\infty}^{\infty}  \sum_{j=1}^3 \prescript{}{j}{\LB}_{ms}^{t\sigma_j}\big(L,{\mathbf b}\big)  \, \fB_{ms}^{\sigma_j} \, ds, \label{true_form}
\end{align}
where ${\mathbf b}$ is the $(2M+1)\times(2N+1)$-dimensional vector of coefficients of $f$, $m$ is now summed over all integers, $s$ is continuous and integrated over the real line, $\sigma_1$, $\sigma_2$, and $\sigma_3$ (the localization values) are any real values satisfying $t<\sigma_1<2$, $t-1<\sigma_2<0$, and $\sigma_3 = t$, and $\prescript{}{j}{\LB}_{ms}^{t\sigma_j}$ are functions taking a lower-triangular matrix and a set of filter coefficients, and returning the coefficient of $\fB_{ms}^{\sigma_j}$ in the expansion of the transformed filter. Here, the integral is equivalent to an inverse Mellin transform with frequency variable $s$, and the bounds on $\sigma_1$ and $\sigma_2$ are necessary to ensure invertibility \cite{vilenkin1978special, vilenkin1991representation}.

Obviously, the infinite summation and the integration in Equation~(\ref{true_form}) make evaluation unfeasible. We propose a practical implementation by truncating the summation over the angular frequency $m$, and replacing the integration over the real line with a discrete approximation using quadrature. The summation is a result of the addition theorem for Bessel functions \cite{watson1995treatise} which appear in the derivation of $\prescript{}{j}{\LB}_{ms}^{t \sigma_j}$; it converges rapidly at low frequencies and can be well-approximated with only several terms \cite{chirikjian2016harmonic}.  The use of quadrature is motivated by the observation that for a fixed transformation $L$ and filter coefficients $\mathbf{b}$ the function $\prescript{}{j}{\LB}_{ms}^{t\sigma_j}$ tends to be smooth and falls off quickly away from $s=0$. Using the approximation, we get
\begin{align}
   L \, f  \approx \sum_{m = -M'}^{M'} \sum_{q=1}^Q  \sum_{j=1}^3 w_q \, \prescript{}{j}{\LB}_{ms_q}^{t\sigma_j}\big(L,{\mathbf b}\big)  \, \fB_{ms_q}^{\sigma_j} \label{approx}
\end{align}
where $\{s_q\}\subset\R$ are the quadrature points and $\{w_q\}$ are the weights.

We remark that the principle idea behind the expansion in Equation~(\ref{true_form}) involves exploiting the symmetry of the spherical log-polar basis functions under \mob{} transformations taking $z$ to $-z^{-1}$. This allows us to replace the projective action of $\OPS{}$ with the affine action of the upper-triangular matrices -- the group of rotations, translations, and dilations -- whose representations are better understood \cite{vilenkin1978special, vilenkin1991representation}.

\subsubsection*{Efficient \mob{} Convolutions} Plugging the approximation in Equation~(\ref{approx}) into the definition of \mob{} convolution in Equation~(\ref{m_conv}) and moving the sums outside the integral gives
\begin{align}
(\feat{} * f)
&\approx
\sum\limits_{\substack{-M'\leq m\leq M'\\1\leq q \leq Q\\1\leq j\leq 3}}
\left(\rho_{\feat{}}\, w_q\, \prescript{}{j}{\LB}_{ms_q}^{t\sigma_j}\big(\T_{\feat{}}, {\mathbf b}\big) \,\, {*}_{e} \,\, \fB_{ms_q}^{\sigma_j} \right).
\label{conv_sums}
\end{align}
Thus, by approximating the pointwise action of the frame operator $\T_{\feat{}}(z)$, we can approximate an arbitrary \mob{} convolution as a sum of identity convolutions.  The components of $\prescript{}{j}{\LB}_{ms_q}^{t\sigma_j}$ depending only on $s$ can be pre-computed for a fixed set of quadrature points so that, in practice, the complexity of evaluating the function at run time is linear in the coefficients of $\mathbf b$.

\subsection{Complexity}
In the approximation of \mob{} convolution in Equation~(\ref{conv_sums}), the right side of the identity convolutions are independent of the filter coefficients, so the innermost bracketed term in Equation~(\ref{e_conv_decomp}) can be pre-computed for a given band-limit $B$ (for every angular frequency $m$, quadrature point $s_q$, and localization index $\sigma_j$). Thus, the $O(B^4)$ computational bottle-neck in computing the identity convolution need only be performed once and the total complexity of computing the \mob{} convolution is $O(M' Q B^3 \log B)$. In applications, we find that setting $M' = M + 1$ and using a $Q=30$ point trapezoidal quadrature rule in Equation~(\ref{approx}) allows us to both suitably approximate the transformation of the filter and scale up to $B = 64$.

%% file: network.tex
Constructing \mob{}-equivariant spherical CNNs with \mob{} convolutions is straight-forward and requires no specialized architecture. The atomic units are the same as those found in regular CNNs -- a convolutional layer, followed by normalization and a non-linearity.

\subsection{Convolutional layers}
For a \mob{} convolution layer mapping $C$-channel input features $\feat{}\in L^2(\RS,\R^C)$ to $C'$-channel output features $\feat{}'\in L^2(\RS,\R^{C'})$, the $c'-$th output feature $\feat{}_{c'}'$ is computed in the usual manner by summing the convolutions of the input features with the filters in the $c'-$th row of the bank. However, the structure of Equation~(\ref{conv_sums}) allows us to preform the reduction over the input channels \textit{before} computing the convolutions in the sum, such that
\begin{align}
    \feat{}_{c'}' = \sum\limits_{\substack{-M'\leq m\leq M'\\1\leq q \leq Q\\1\leq j\leq 3}}
 \left( \, \sum_{c=1}^{C} \,  \h_{\feat{}_c} \, w_q \, \prescript{}{j}{\LB}_{ms_q}^{ t\sigma_j}\big(\T_{\feat{}_c},{\mathbf b}^{cc'}\big) \, \, {*}_{e} \, \, \fB_{ms_q}^{\sigma_j} \, \right),  \label{conv_reduction}
\end{align}
where ${\mathbf b}^{cc'}$ denotes the $(2M + 1)(2N + 1)$ parameters for the $(c,c')$-th filter in the bank. Thus, for each convolutional layer mapping $C$ input features to $C'$ output features, we only need to compute $C'$ \mob{} convolutions instead of $C \times C'$.

This advantage is not without caveat. As discussed in the supplement (\S\ref{a:overhead}), a naive implementation of the inner sum over the input channels produces large intermediate tensors at high resolutions ($B \geq 64$), which can quickly fill GPU memory . Our layers are implemented in PyTorch \cite{paszke2019pytorch}, where we fuse this operation to reduce its overhead.

\subsection{Normalization and Non-linearities}
Standard normalization techniques don't commute with \mob{} transformations, since the mean and standard deviation of spherical signals are not invariant under dilations. Instead, we introduce a conformally-equivariant normalization layer based on filter response normalization \cite{singh2020filter}, replacing the square mean with the Dirichlet energy which \textit{is} invariant under \mob{} transformations. The normalization is applied on a per-channel basis independent of the batch size via the mapping 
\begin{align}
    \feat{}_c \mapsto \frac{ \alpha_c \, \feat{}_c}{  \sqrt{ \bigints_{\RS}  \h_{\feat{}}(z)  \, dz + \epsilon_c}} + \beta_c, \label{fr_d}
\end{align}
where $\h_{\feat{}}$ is defined as in Equation~(\ref{deriv_density}) and $\alpha_c, \, \beta_c \in \mathbb{R}$ and $\epsilon_c \in \mathbb{R}_{>0}$ are learnable per-channel parameters.  

Following normalization we apply thresholded activations as non-linearities, which have been shown to better compliment filter response normalization than other activation layers \cite{singh2020filter}. Here, we replace the ReLU with the Mish activation \cite{misra2019mish} which we find improves training speed and performance. Specifically, non-linearities are applied pointwise as,
\begin{align}
    \feat{}_c \mapsto \textrm{Mish}\big(\feat{}_c - \gamma_c\big)  + \gamma_c, \label{nonlin}
\end{align}
where $\gamma_c \in \mathbb{R}$ is a learnable per-channel threshold value. We note that the thresholded activation is not fundamental to our framework, and can be replaced with other activation layers if desired.

%% file: evaluation.tex
We validate our claim of \mob{}-equivariance empirically and demonstrate the utility of \mob{}-equivariant CNNs by achieving strong results in both geometry and spherical-image processing tasks. In the former paradigm, we perform genus-zero shape classification by conformally mapping surfaces to the sphere; in the latter, we consider the task of omni-directional image segmentation. 

Our principal module in applications is an MCResNet block \cite{he2016deep} , consisting of two \mob{} convolutions, each followed by the normalization layer and non-linearity described in Equations~(\ref{fr_d}-\ref{nonlin}), with a residual connection between the input and output streams. We use $M=N=1$ band-limited filters and set $t = 0.15$, optimizing for $\sigma_1$ and $\sigma_2$ and the quadrature points $\{s_q\}$ as described in the supplement (\S\ref{a:quad}). Our framework is implemented in PyTorch \cite{paszke2019pytorch}, and we fit our networks using SGD with Nesterov momentum \cite{sutskever2013importance}, training for $60$ epochs with an initial learning rate of $10^{-2}$, decaying to $10^{-4}$ on a cosine annealing schedule \cite{loshchilov2017sgdr}. 
%%%%%%%%%%%%%%%%%%%%%%%%%%%%%%%%%%%%%%%%%%%%%%%%%%%%%%
%%%%%%%%%%%%%%% Equivariance %%%%%%%%%%%%%%%%%%%%%%%%%
%%%%%%%%%%%%%%%%%%%%%%%%%%%%%%%%%%%%%%%%%%%%%%%%%%%%%
\subsection{Equivariance}
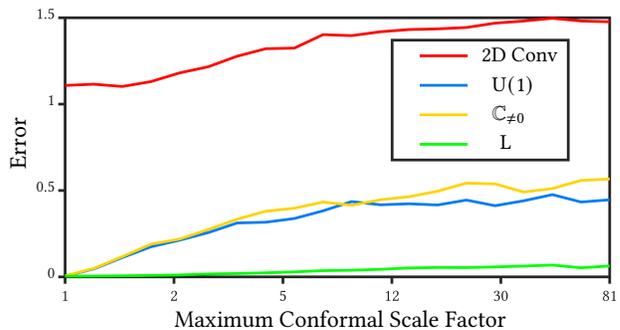
\begin{figure}[t]
\begin{picture}(0.925\linewidth,0.52\columnwidth)
\put(5,10){{\includegraphics[width=0.95\linewidth]{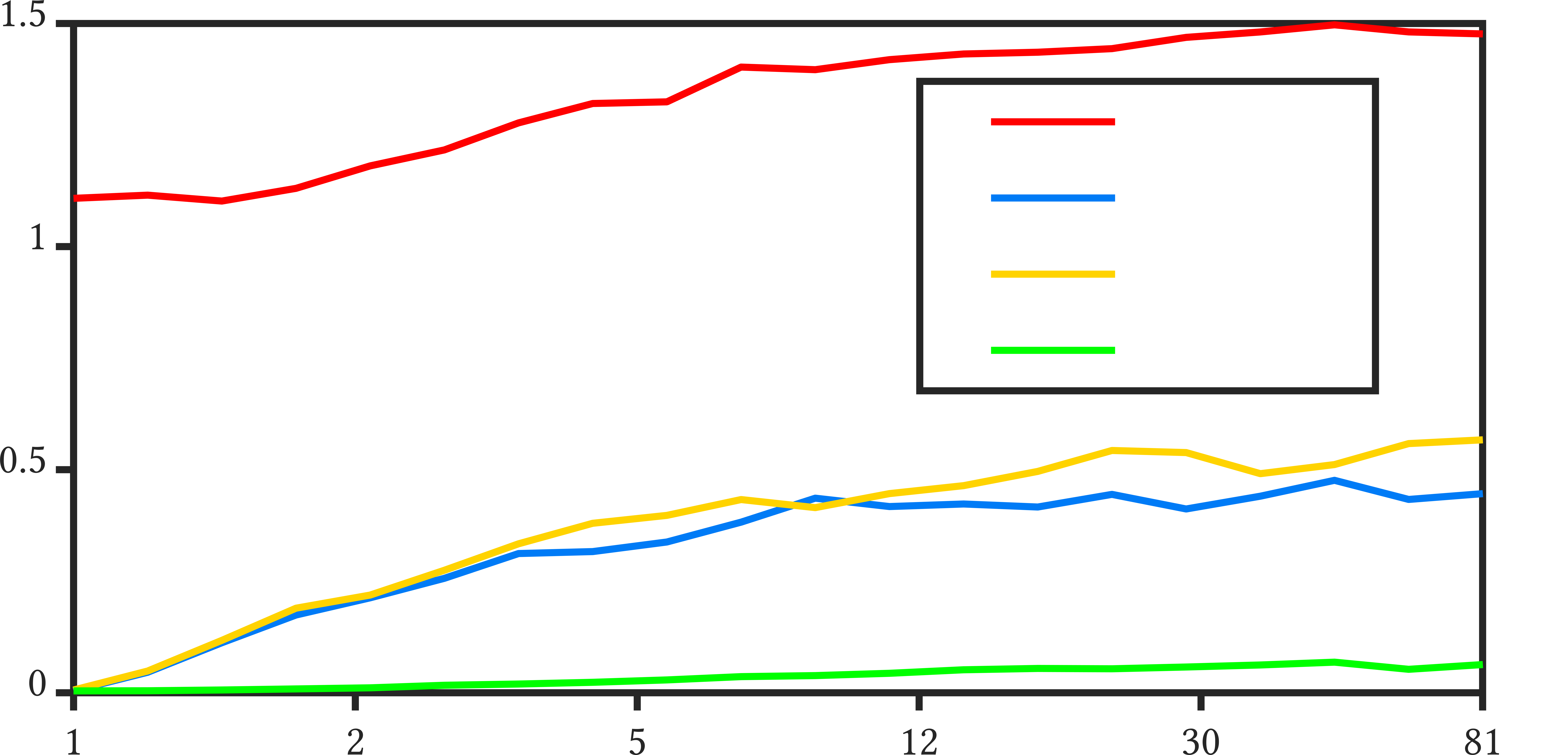}}}
%\put(5, 10){\usebox{\plotbox}}
\put(57, 0){Maximum Conformal Scale Factor}
\put(-5, 60){\rotatebox{90}{Error}}
\put(171, 100.5){ {\small 2D Conv}}
\put(175, 89.5){ {\small $\textrm{U}(1)$}}
\put(175, 78.5){ {\small $\mathbb{C}_{\neq 0}$}}
\put(178, 67.5){ {\small $\textrm{L}$}}
\end{picture}
    \caption{ The equivariance error plotted as a function of the maximum conformal scale factor. Notably, moving from $\textrm{U}(1)$ (rotations) to $\mathbb{C}_{\neq 0}$ (rotations and dilations) does not provide a benefit -- one must consider the full lower-triangular subgroup $\OPS{} \subset \SLC$ .        \label{equi_plot}
 }
%\vspace{-.15in}
\end{figure}

We empirically validate the equivariance of our framework by quantifying the degree to which our layers commute with \mob{} transformations of increasing area distortion. We consider a $32$-channel, $B=64$ band-limited MCResNet block with the equivariant residual connection removed to avoid bias. We control the area distortion of a \mob{} transformation $g$ by composing a series of random rotations and inversions so that the maximal scale factor over $\RS$ equals a fixed value \cite{baden2018mobius}. Denoting $\mathcal{R}$ as the mapping induced by passing features through the MCResNet layer, we follow \cite{de2020gauge, worrall2019deep, sosnovik2019scale} and define the equivariance error for a fixed maximum scale factor $\alpha \in \mathbb{R}_{\geq 0}$ as
\begin{align}
    \textrm{Error} = \frac{\textrm{E}\left( \, \mathcal{R}( \, g \psi \,) - g \, \mathcal{R}(\, \psi \,) \right)^2}{\textrm{Var}  \ g \, \mathcal{R}( \,\psi \,)} \quad \textrm{with} \quad \max_{z \in \RS} \scalef{g}{z} = \alpha,\label{equiv_error}
\end{align}
where $\textrm{E}$ and $\textrm{Var}$ denote the mean and variance computed over $100$ randomly initialized models, \mob{} transformations, and features. 

As a baseline, we compare our proposed approach against three other paradigms. In the first, we replace \mob{} convolution with a standard $5 \times 5$ convolution layer taking $\h_{\psi}$ as input; in the second, we restrict the transformation field to rotations so that $\T_{\feat{}}(z) \in \textrm{U}(1)$; in the third, we loosen the restriction to include dilations with $\T_{\feat{}}(z) \in \mathbb{C}_{\neq 0}$. We note that the second and third paradigms are isometry-equivariant, and that the latter is also equivariant to the conformal transformations of the (non-compactified) plane.

The results are shown in \Fig{equi_plot}, where the equivariance error in Equation~(\ref{equiv_error}) is plotted as a function of the maximum scale factor. The green curve is our proposed method with $\T_{\feat{}}(z) \in \OPS{}$. Using our method, the error stays very low, indicating that \mob{} convolution is approximately equivariant even in the presence of significant changes in scale ($\scalef{g}{z} \geq 12)$.  Notably, we see no improvement moving from $\T_{\feat{}}(z) \in \textrm{U}(1)$ to $\T_{\feat{}}(z) \in \mathbb{C}_{\neq 0}$, suggesting that rotations and dilations alone fail to well-characterize the local deformations induced by \mob{} transformations.

%%%%%%%%%%%%%%%%%%%%%%%%%%%%%%%%%%%%%%%%%%%%%%%%%%%%%%
%%%%%%%%%%%%%%% Conf. Meshes Example %%%%%%%%%%%%%%%%%%%%%%%%%
%%%%%%%%%%%%%%%%%%%%%%%%%%%%%%%%%%%%%%%%%%%%%%%%%%%%%
\subsection{Conformal Shape Classification}
\label{sec:classification}

\begin{table}[t]{
\caption{ Genus-zero shape classification. Several conformally deformed meshes from the SHREC '11 dataset \shortcite{lian2011} are shown below. \label{tab:shape_class}}
\begin{picture}(\columnwidth,0.55\columnwidth)
\put(7.5, 105){
    \begin{tabular}{llc}
    Method                           & \ Orig. Accuracy         & \ Conf. Accuracy \\ %\hline \hline
    \cline{1-3} \noalign{\vskip\doublerulesep \vskip-0.5\arrayrulewidth} \cline{1-3}
    \multicolumn{1}{l|}{MC (ours)}         & \multicolumn{1}{c|}{99.1\%}  & 86.5\%  \\
    \hline
    \multicolumn{1}{l|}{DiffusionNet \shortcite{sharp2020diffusion}}     & \multicolumn{1}{c|}{99.5\%}  & 64.9\%  \\
     \multicolumn{1}{l|}{FC \shortcite{Mitchel_2021_ICCV}}     & \multicolumn{1}{c|}{99.2\%}  & 40.7\%  \\
      \multicolumn{1}{l|}{CubeNet \shortcite{shakerinava2021equivariant}}     & \multicolumn{1}{c|}{47.1\%}  & 21.2\%  \\
    \end{tabular}
    }
    \put(34, 0){{\includegraphics[width=0.8\columnwidth]{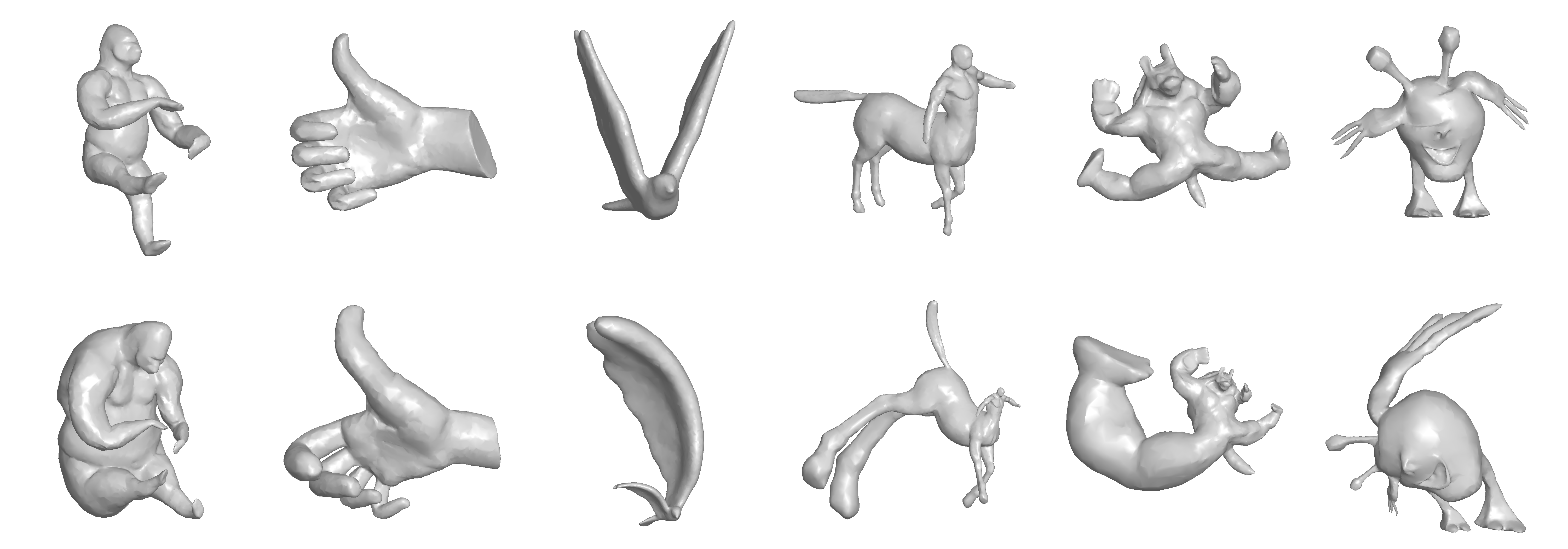}}}
    \put(10, 13){Conf.}
    \put(10, 50){Orig.}
     \end{picture} 
}
\end{table}

Next, we use \mob{} convolutions to classify genus-zero shapes. Over the last half-decade, the SHREC '11 dataset \cite{lian2011} has become a popular choice for evaluating network performance in shape classification tasks and several recent approaches have achieved near-perfect accuracy on the dataset \cite{wiersma2021deltaconv, Mitchel_2021_ICCV, milano2020primaldual, sharp2020diffusion}. However, shapes within each of the 30 categories in the dataset differ only by (approximate) isometric deformations. To better highlight the strengths of our approach, we extend the dataset to include deformations given by random conformal transformations with several examples shown above Table~\ref{tab:shape_class}.

To apply our framework, we conformally map each mesh to the sphere via mean curvature flow \cite{kazhdan2012can} and use a simple network consisting of a single $16$-channel, $B = 64$ band-limited MCResNet block followed by a global mean pool and a fully-connected layer to give predictions for the 30 shape categories.

We fit our network on both the original SHREC '11 dataset and our conformally-augmented version using 10 samples per class. For comparisons, we report the results of Field Convolutions (FC) \cite{Mitchel_2021_ICCV}  and DiffusionNet \cite{sharp2020diffusion} -- two state-of-the-art surface networks -- on the original dataset and train both networks on the conformally-augmented version. We also train CubeNet \cite{shakerinava2021equivariant}, a top performing rotation-equivariant spherical network, on both datasets. As inputs, each network takes the Heat Kernel Signature (HKS) \cite{sun2009concise} computed at $16$ different timescales; since the HKS isn't conformally-invariant, we use the values computed on the original meshes when training and testing on their conformally-augmented counterparts. 

Results are shown in Table~\ref{tab:shape_class} in the form of the mean classification accuracy over three randomly sampled test-train splits. Our simple \mob{} convolution network matches the state-of-the-art performance of FC and DiffusionNet on the original dataset and significantly outperforms both on the conformally augmented version, despite the fact that the transformations between the spherically-parameterized meshes aren't perfect \mob{} transformations. Like FC and DiffusionNet, our method is equivariant to isometric deformations of the meshes as they manifest as \mob{} transformations after parameterization to the sphere, serving to explain our strong performance on the original dataset. However, \mob{}-equivariance allows our rudimentary network to better account for conformal deformations between similar shapes and suggests that a new class of  \textit{conformally-equivariant} surface networks may outperform existing \textit{isometry-equivariant} networks in challenging shape analysis and recognition tasks. CubeNet performs poorly on both datasets, indicating that \mob{}-equivariance isn't easily learned in a rotation-equivariant framework.

\subsection{Omni-directional Image Segmentation}
\label{segmentation_section}
\begin{table}[t]{
    \caption{ Omni-directional image segmentation \label{segmentaion_results}}

    \begin{tabular}{llcc}
    Method                           & \ Accuracy          & \ IoU \\ \cline{1-3} \noalign{\vskip\doublerulesep \vskip-0.5\arrayrulewidth} \cline{1-3}
    \multicolumn{1}{l|}{MC (ours)}         & \multicolumn{1}{c|}{60.9\%}  & 43.3\%  &\rdelim\}{3}{1cm}[\, Spectral] \\
    \cline{1-3}
\multicolumn{1}{l|}{SWSCNN~\shortcite{esteves2020spin}}     & \multicolumn{1}{c|}{58.7\%}  & 43.4\%  \\
    \multicolumn{1}{l|}{SphCNN~\shortcite{esteves2018learning}}     & \multicolumn{1}{c|}{52.8\%}  & 40.2\%  \\
    \multicolumn{1}{l|}{}         & \multicolumn{1}{c|}{}  & \\[-10pt]
    \cdashline{1-3} 
    \multicolumn{1}{l|}{}         & \multicolumn{1}{c|}{}  & \\[-10pt]
    \multicolumn{1}{l|}{CubeNet~\shortcite{shakerinava2021equivariant}}         & \multicolumn{1}{c|}{62.5\%}  & 45.0\%  & \rdelim\}{3}{1cm}[\, Spatial]\\
    \multicolumn{1}{l|}{HexNet~\shortcite{zhang2019orientation}}         & \multicolumn{1}{c|}{58.6\%}  & 43.3\% \\
    \multicolumn{1}{l|}{UGSCNN~\shortcite{jiang2018spherical}} & \multicolumn{1}{c|}{54.7\%} &38.3\% 
    \end{tabular}
     
    }
\end{table}

Last, we demonstrate the utility of \mob{} convolutions by moving from geometry to image processing, where we apply them to semantically segment omni-directional images from the Stanford 2D3DS dataset \cite{armeni2017joint}. Here, we use MCResNet blocks to construct a U-Net \cite{ronneberger2015unet} architecture with $32, 64, 128, 256, 128, 64, 32$ channels per layer, applying max pooling or nearest neighboring upsampling before each increase or decrease in channel width. As with other state-of-the-art equivariant spherical networks \cite{esteves2020spin, shakerinava2021equivariant}, we find our method performs best as a feature extractor for a small network of standard convolutional layers due to the consistent latitudinal orientation of the images; we append six $3 \times 3$ 2D convolutions to the end of our network to predict labels.  To measure performance, we report the mean per-class segmentation accuracy and intersection over union (IoU) averaged over the three official folds.

Results are shown in Table~\ref{segmentaion_results}, and we attain performance comparable to the state-of-the-art. Existing spherical networks compute convolutions either in the spatial domain \cite{shakerinava2021equivariant, zhang2019orientation, jiang2018spherical} or, like our method, in the spectral domain via expansions in spherical basis functions. In the latter case, efficiency and scalable filter support come at the cost of fidelity, as some high-frequency information is lost when computing the forward SHT due to the fixed band-limit assumption. This puts spectral methods at a disadvantage in precision labeling tasks with a large class imbalance, where spectral aliasing can blur sharp boundaries and over-smooth localized features necessary to make accurate predictions. Compounded by the devaluation of equivariance due to the consistent orientation of the images, this is a challenging task for our framework. However, we outperform existing rotation-equivariant spectral approaches, demonstrating that we are able to achieve \mob{}-equivariance  without sacrificing descriptiveness. A visualization of our results and further discussion are provided in the supplement (\S\ref{a:seg}).

%% file: conclusion.tex
We present a novel, \mob{}-equivariant spherical convolution operator which we call \textit{\mob{} convolution}. With it, we develop the foundations for \mob{}-equivariant spherical CNNs and demonstrate the utility of this framework by achieving strong results in both geometry and spherical-image processing tasks.   

More generally, this work represents an effort to move both image and surface convolutional neural networks beyond standard rotation- and isometry-equivariance and into the realm of conformal-equivariance. In particular, our experiments suggest that the latter transition may be especially relevant in the context of shape analysis and recognition and we hope this work serves to catalyze the development of a new generation of surface networks better able to handle the kinds of complex deformations found in real-world shape data.

%% file: supp/supp_intdiff.tex
We view $\RS{}$ as a Riemannian manifold under the round metric $s$, expressed at $z = z_1 + i \, z_2 \in \RS$ as
\begin{align}
s_{z} = \frac{4}{\left(1 + \lvert z \rvert^2 \right)^2} \, (dz_1^2 + dz_2^2). \label{round_metric}
\end{align}

\subsection{Integration}
The area element is 
\begin{align}
    dz = \frac{4 \, dz_1 \, dz_2}{\left(1 + \lvert z \rvert^2 \right)^2}, \label{area_element}
\end{align}
and for $g=\left[\begin{smallmatrix}a&b\\c&d\end{smallmatrix}\right]\in\SLC$, the scale factor $\scalef{g}{z}$ associated with the change of variables 
\begin{align}
\begin{aligned}
    z & \mapsto gz \\
    dz &\mapsto \scalef{g}{z} \, dz
\end{aligned} \label{c_of_v}
\end{align}
is given by \cite{carmeli2000group}:
\begin{align}
    \scalef{g}{z} = \frac{(1 + \lvert z \rvert^2 )^2}{(1 + \lvert g z \rvert^2)^2 \, \lvert c z + d \rvert^4} = \frac{(1 + \lvert z \rvert^2 )^2}{\left( \, \lvert a z + b \rvert^2 + \lvert c z + d \rvert^2 \right)^2}. \label{scale_factor}
\end{align}

\subsection{Differentiation}
For any $\feat{} \in L^2(\RS, \mathbb{R})$, we express the differential of $\feat{}$ with respect to the complex variable $z = z_1 + i \, z_2$ as the complex number
\begin{align}
    d \, \feat{} = \frac{1}{2}\left( \frac{\partial \feat{}}{\partial z_1} - i \, \frac{\partial \feat{}}{\partial z_2}\right). \label{diff}
\end{align}
Similarly, we can define the Hesssian of $\feat{}$ as a complex number whose coefficients are related to the covariant Hessian. The terms depend on the first and second partial derivatives of $\feat{}$ in addition to the Christoffel symbols corresponding to the round metric. At the origin, the terms depending on the first derivatives and Christoffel symbols vanish, and the Hessian becomes
\begin{align}
    \nabla d \, \feat{}\big|_0 = \frac{1}{4} \eval{\left( \frac{\partial^2 \feat{}}{d z_1^2} - \frac{\partial^2 \feat{}}{\partial z_2^2} - 2 i \, \frac{\partial^2 \feat{}}{\partial z_1 \partial z_2} \, \right)}{0}. \label{hess}
\end{align}
Given an element $g=\left[\begin{smallmatrix} a & b \\ c & d \end{smallmatrix}\right]\in \SLC$, the differential of $g$ at $x \in \RS$ and the Hessian of $g$ at the origin are the complex numbers
\begin{align}
    d \, g\big|_x = \frac{1}{(cx + d)^2} \qquad \textrm{and} \qquad  \nabla d \, g\big|_0 = -\frac{2c}{d^3}. \label{g_diffs}
\end{align}

%% file: supp/supp_equiv.tex
Here we provide detailed proofs of the claims made in Section~\ref{s:conv} regarding sufficient conditions for \mob{}-equivariance and the construction of the frame and density operators.

\subsection{Conditions for Equivariance}
Suppose we are given a frame operator $\T$ and a density operator $\h$ taking a real-valued function and returning a frame-field and density, respectively
\begin{align*}
\begin{aligned}
\T:L^2(\RS,\R) &\rightarrow L^2(\RS, \, \textrm{L}) \\
\feat{} & \mapsto \T_{\feat{}}
\end{aligned}
\quad \textrm{and} \quad 
\begin{aligned}
\h: L^2(\RS,\R) &\rightarrow L^2(\RS,\R) \\
\feat{} & \mapsto \h_{\feat{}}
\end{aligned}.
\end{align*}

\begin{claim}
If for all $\feat{} \in L^2(\RS, \mathbb{R})$ and $g \in \textup{SL(2, \ensuremath{\mathbb{C}})}$, the operators $\T$ and $\h$ satisfy the condition of Equation~(\ref{hT_cond})
\begin{align*}
\lnxf{g}{z} \, \T_{\feat{}}(z) = \T_{g \, \feat{}}(gz)
\qquad \textrm{and} \qquad 
\lambda_g^{-2}(z) \, \h_{\feat{}}(z) = \h_{g \, \feat{}}(gz),
\end{align*} 
then for any filter $f \in L^2(\RS, \mathbb{R})$,
\begin{align*}
   g \, (\feat{} * f) = (g \, \feat{} * f). 
\end{align*}
\end{claim}

\begin{proof}
Suppose $\T$ and $\h$ satisfy the condition and consider any $\feat{} \in L^2(\RS, \mathbb{R})$ and $g \in \SLC$. For any filter $f \in L^2(\RS, \mathbb{R})$, we can relate the expression of the filter over $\RS_{z}$ to the expression of the filter over $\RS_{gz}$:
\begin{align}
\nonumber
\big[\T_{\feat{}}(z)  \, f \big] \circ \log_z 
& \stackrel{\phantom{(\ref{hT_cond})}}{=} f \circ \big[\T_{\feat{}}(z)\big]^{-1} \, \log_z \\
\nonumber
&\stackrel{(\ref{hT_cond})}{=} f \circ \big[\T_{g \, \feat{}}(gz)\big]^{-1} \, \lnxf{g}{z} \, \log_z  \\
\nonumber
&  \stackrel{\stackrel{\phantom{(\ref{hT_cond})}}{(\ref{param_xform})}}{=} f \circ \big[\T_{g \, \feat{}}(gz)\big]^{-1} \, \log_{gz} \, g  \\
\label{framed_filter}
&\stackrel{\phantom{(\ref{hT_cond})}}{=} \big[\T_{ g \, \feat{}}(gz)  \, f \big] \circ \log_{gz} \, g
\end{align}
Using the relationship between the expression of the filters over $\RS_z$ and $\RS_{gz}$ it follows that for any $y \in \RS$,
\begin{align*}
    (\feat{} * f)( g^{-1} y) &  \stackrel{(\ref{m_conv})}{=} \int_{\RS} \h_{\feat{}}(z) \,
    \big[\T_{\feat{}}(z)  \, f \big]\big( \log_z g^{-1} y \big) \ dz \\
    & \stackrel{(\ref{framed_filter})}{=} \int_{\RS} \h_{\feat{}}(z) \,
    \big[\T_{ g \, \feat{}}(gz)  \, f \big]\big( \log_{gz} y \big) \ dz \\
    & \stackrel{(\ref{hT_cond})}{=} \int_{\RS} \scalef{g}{z} \, \h_{g \, \feat{}}(g z) \,
    \big[\T_{ g \, \feat{}}(gz)  \, f \big]\big( \log_{gz} y \big) \ dz \\
    & \stackrel{\phantom{(\ref{hT_cond})}}{=} \int_{\RS}  \h_{g \, \feat{}}(z') \,
    \big[\T_{ g \, \feat{}}(z')  \, f \big]\big( \log_{z'} y \big) \ dz' \\
    &  \stackrel{(\ref{m_conv})}{=} (g \, \feat{} * f)(y),
\end{align*}
where the second to last equality follows from the change of variables in Equation~(\ref{c_of_v}).
\end{proof}

\subsection{Construction of Operators}

\begin{claim}
If $\T$ is defined as in Equation~(\ref{deriv_frame})
\begin{align*}
    \T_{\psi}(x) &\equiv \begin{bmatrix} \left( d \log_{x}\feat{} \big\vert_{0}\right)^{-\frac{1}{2}}  & 0 \\
    \left(\frac{1}{2} \nabla d  \log_{x} \feat{} \big\vert_{0}\right) \left( d \log_{x}\feat{}  \big\vert_{0}\right)^{-\frac{3}{2}}& \left( d \log_{x}\feat{} \big\vert_{0}\right)^{\frac{1}{2}}\end{bmatrix}
\end{align*}
then the condition for the frame operator in Equation~(\ref{hT_cond}) is satisfied.
\end{claim}

\begin{proof}
For any $\feat{} \in L^2(\RS, \mathbb{R})$ and $g \in \SLC$ it follows from Equation~(\ref{param_xform}) that 
\begin{align*}
\log_{gx}  g  \, \feat{} = \lnxf{g}{x} \log_{x} \feat{}.
\end{align*}
Denoting $\lnxf{g}{x} = \left[\begin{smallmatrix} a & 0 \\ n & a^{-1} \end{smallmatrix} \right]$, applying the chain rule and evaluating at the origin using Equation~(\ref{g_diffs}) gives
\begin{align}
    \big[ d \log_{gx} g \, \feat{} \big] \big\vert_{0} & = a^{-2}  \big[ d \log_{x}\feat{} \big] \big\vert_{0}, \label{chain_diff}
\end{align}
and
\begin{align}
    \big[\nabla d  \log_{gx} g \, \feat{} \big] \big\vert_{0} & = a^{-4} \big[\nabla d  \log_{x} \feat{} \big] \big\vert_{0} + 2na^{-3}  \big[ d \log_{x}\feat{} \big] \big\vert_{0}. \label{chain_hess}
\end{align}
The upper diagonal element of $\T_{g \, \feat{}}(gx)$ is given by
\begin{align}
    \big[\T_{g \, \feat{}}(gx)\big]_{11} & \stackrel{(\ref{deriv_frame})}{=} \left(\big[ d \log_{gx} g \, \feat{} \big] \big\vert_{0}\right)^{-\frac{1}{2}} \nonumber \\
    &\stackrel{(\ref{chain_diff})}{=} a  \left(\big[ d \log_{x}\feat{} \big] \big\vert_{0}\right)^{-\frac{1}{2}} \nonumber \\
    & \stackrel{(\ref{deriv_frame})}{=} a \big[\T_{\feat{}}(x)\big]_{11}, \nonumber 
\end{align}
A similar argument shows that the lower diagonal element satisfies $\big[\T_{g \, \feat{}}(gx)\big]_{22} = a^{-1} \big[\T_{\feat{}}(x)\big]_{22}$. For the off-diagonal element have
\begin{align*} 
 \big[\T_{g \, \feat{}}(gx)\big]_{21} & \stackrel{(\ref{deriv_frame})}{=} \frac{1}{2} \big[\nabla d  \log_{gx} g \, \feat{} \big] \big\vert_{0} \left(\big[ d \log_{gx} g \, \feat{} \big] \big\vert_{0}\right)^{-\frac{3}{2}} \nonumber \\
 & \stackrel{(\ref{chain_hess})}{=}  n \left(\big[ d \log_{x}\feat{} \big] \big\vert_{0}\right)^{-\frac{1}{2}}  +  \\
 & \phantom{\stackrel{(\ref{chain_hess})}{=}}  \ \ a^{-1} \frac{1}{2} \big[\nabla d  \log_{x} \feat{} \big] \big\vert_{0} \left(\big[ d \log_{x}\feat{} \big] \big\vert_{0}\right)^{-\frac{3}{2}} \nonumber \\
  & \stackrel{(\ref{deriv_frame})}{=} n \big[\T_{\feat{}}(x)\big]_{11} +  a^{-1} \big[\T_{\feat{}}(x)\big]_{21} \nonumber 
\end{align*}
from which it follows that 
$$
\lnxf{g}{x} \T_{\feat{}}(x) = \T_{g \, \feat{}}(gx)
$$
as desired.
\end{proof}

\begin{claim}
If $\h$ is defined as in Equation~(\ref{deriv_density})
\begin{align*}
    \h_{\psi}(x) &\equiv \abs[\Big]{ \, d \log_{x}\feat{} \big\vert_{0} \, }^{\, 2}
\end{align*}
then the condition for the density operator in Equation~(\ref{hT_cond}) is satisfied. 
\end{claim}

\begin{proof}
From Equations~(\ref{gen_log})~and~(\ref{g_diffs}), the differential of the generalized exponential $\exp_{x} = \log_{x}^{-1}$ at the origin is given by
\begin{align*}
    d \exp_{x}\big|_0 = \frac{\abs[]{c}^2 (1 + \abs[]{x}^2)}{c^2},
\end{align*}
from which it follows that 
\begin{align}
\abs[\Big]{\, d \exp_{x}\big|_0 \, } = (1 + \abs[]{x}^2) \label{dlog_abs}
\end{align}
for any choice of $c \in \mathbb{C}$. Then, for any $g \in \SLC$ applying the chain rule to the definition of  $\lnxf{g}{x}$ in Equation~(\ref{param_xform}) gives
\begin{align}
   \abs[\Big]{ \, d \lnxf{g}{x}\big\vert_{0} \, }^{\, 2} & \stackrel{(\ref{param_xform})}{=}  \abs[\Big]{ \, d \exp_{x} \big\vert_{0} \, }^{\, 2} \  \abs[\Big]{ \, d g\big\vert_{x} \, }^{ \, 2} \ \abs[\Big]{ \, d \log_{gx}\big\vert_{gx} \, }^{\, 2} \nonumber \\
   & \stackrel{\phantom{(\ref{param_xform})}}{=} \abs[\Big]{ \, d \exp_{x} \big\vert_{0} \, }^{\, 2}  \ \abs[\Big]{ \, d g \big\vert_{x} \, }^{ \, 2} \ \abs[\Big]{ \, d \exp_{gx} \big\vert_{0} \, }^{ \, -2}  \nonumber \\
  & \stackrel{(\ref{dlog_abs})}{=} (1 + \abs[]{x}^2)^2 \, \left(\frac{1}{\abs[]{cx + d}^4}\right) \,\left(\frac{1}{(1 + \abs[]{gx}^2)^2}\right) \nonumber \\
  &\stackrel{\phantom{(\ref{dlog_abs})}}{=} \frac{(1 + \abs[]{x}^2)^2}{(1 + \abs[]{gx}^2)^2 \abs[]{cx + d}^4} \nonumber \\
  & \stackrel{(\ref{scale_factor})}{=} \ \scalef{g}{x} \label{diff_Dgz},
\end{align}
where the second equality follows from the fact that $\log_z$ is an isometry of $\RS$ with $\log_z^{-1} = \exp_z$.

It follows that for any $\feat{} \in L^2(\RS, \mathbb{R})$ and $g \in \SLC$, 
\begin{align*}
    \h_{g \, \feat{}}(gx) & \stackrel{(\ref{deriv_density})}{=} \abs[\Big]{ \,  d \log_{gx} g \, \feat{} \big\vert_{0} \, }^{\, 2} \\
    &  \stackrel{(\ref{param_xform})}{=} \abs[\Big]{ \, d \lnxf{g}{x} \log_x  \feat{} \big\vert_{0} \, }^{\, 2} \\
    &  \stackrel{\phantom{(\ref{param_xform})}}{=} \abs[\Big]{ \, d \log_x  \feat{}\circ\left(\lnxf{g}{x}\right)^{-1} \Big\vert_{0} \, }^{\, 2} \\
    & \stackrel{\phantom{(\ref{param_xform})}}{=} \abs[\Big]{ \, d \lnxf{g}{x} \big\vert_{0} \, }^{\, -2} \ \abs[\Big]{ \, d \log_x  \feat{} \big\vert_{0} \, }^{\, 2} \\ \nonumber 
    & \stackrel{(\ref{diff_Dgz})}{=} \lambda_{g}^{-2}(x) \ \abs[\Big]{ \,  d \log_x  \feat{} \big\vert_{0} \, }^{\, 2} \\
    & \stackrel{(\ref{deriv_density})}{=} \lambda_{g}^{-2}(x) \, \h_{\psi}(x),
\end{align*}
as desired.
\end{proof}

%% file: supp/supp_xform.tex
Given a filter $f$ parameterized as in Equation~(\ref{filter_param}),
\begin{align*}
    f = \sum_{m=-M}^M\sum_{s=-N}^N  b_{ms} \, \fB_{ms}^t,
\end{align*}
where $\fB_{ms}^t$ are the spherical log-polar functions
\begin{align*}
    \fB_{ms}^t(z) \equiv \frac{|z|^{is}}{|z|^t} \left(\frac{z}{|z|}\right)^m, 
\end{align*}
we derive the expansion of the transformation of the filter by a lower triangular matrix $L = \left[\begin{smallmatrix} a & 0 \\ n & a^{-1}\end{smallmatrix}\right] \in \OPS{} \subset \SLC$ given in Equation~(\ref{true_form}):
\begin{align*}
    L \, f = \sum_{m = -\infty}^{\infty} \int_{-\infty}^{\infty}  \sum_{j=1}^3 \prescript{}{j}{\LB}_{ms}^{t\sigma_j}\big(L,{\mathbf b}\big)  \, \fB_{ms}^{\sigma_j} \, ds. 
\end{align*}
Noting that $\prescript{}{j}{\LB}_{ms}^{t\sigma_j}\big(L,{\mathbf b}\big)$ is a linear function in $\mathbf b$, we first derive an expansion for the transformation of the log-polar bases $L \,  \fB_{ms}^t$. Then, we substitute this expression into the filter parameterization to recover $ \prescript{}{j}{\LB}_{ms}^{t\sigma_j}\big(L,{\mathbf b}\big)$. Afterwards we discuss how we approximate the expansion in practice as in Equation~(\ref{approx}).

\subsection{Transformation of Spherical Log-Polar Bases}
Here we derive an expansion of the transformation of the spherical log-polar bases $\fB_{ms}^t$ by a lower triangular matrix $L \in \OPS{}$. Specifically, we seek an expansion which expresses $L \, \fB_{ms}^t$ as a linear combination of log-polar bases depending on $z$, indexed in angular and radial frequencies $u$ and $\omega$ and localization variables $\sigma$ --$\fB_{u\omega}^\sigma$ -- with a set of coefficient functions  depending \textit{only} on $L$.

We consider elements $L \in \OPS{} \subset \SLC$ of the form 
\begin{align}
L = \begin{bmatrix} a & 0 \\ n & a^{-1} \end{bmatrix} \qquad \textrm{with} \qquad a, n \in \mathbb{C}, \ a \neq 0. \label{L_param}
\end{align}
We treat separately the cases where $n = 0$ and $n \neq 0$, first finding an expansion of $L \, \fB_{ms}^t$ for each and afterwards combining the two to form an expansion of $L \, \fB_{ms}^t$ which holds for all $L \in \OPS{}$.
\newline
\newline
\noindent \textbf{Case 1 $(n = 0$) :}
If $n = 0$, then for any $z \in \RS$
\begin{align*}
L^{-1} z = a^{-2} z,
\end{align*}
and directly evaluating $\big[L \, \fB_{ms}^t\big](z) = \fB_{ms}^t\big(L^{-1}z\big)$ gives
\begin{align}
   \big[L \, \fB_{ms}^t\big](z) &\stackrel{(\ref{log_polar})}{=} \frac{|a^{-2}z|^{is}}{|a^{-2}z|^t} \left(\frac{a^{-2}z}{|a^{-2}z|}\right)^m  \nonumber \\
    &\stackrel{\phantom{(\ref{log_polar})}}{=}  \frac{|a^{-2}|^{is}}{|a^{-2}|^t}  \frac{|z|^{is}}{|z|^t} \left(\frac{a^{-2}}{|a^{-2}|}\right)^m \left(\frac{z}{|z|}\right)^m  \nonumber \\
    &\stackrel{\phantom{(\ref{log_polar})}}{=}  \frac{|a^2|^{-is}}{|a^2|^{-t}} \left(\frac{a^2}{|a^2|}\right)^{-m}  \frac{|z|^{is}}{|z|^t}  \left(\frac{z}{|z|}\right)^m \nonumber \\
    &\stackrel{(\ref{log_polar})}{=} \fB_{-m-s}^{-t}(a^2) \, \fB_{ms}^t(z) \label{case_1_expansion},
\end{align}
where the last equality provides the desired expansion.
\newline
\newline
\noindent \textbf{Case 2 $(n \neq 0)$}: Here, finding an expansion for $L \, \fB_{ms}^t$ is significantly more involved as $L^{-1}$ acts projectively on $\RS$. Specifically,
\begin{align*}
    L^{-1} z = \frac{a^{-1} z}{a - nz} = \frac{z}{a^2 - anz},
\end{align*}
which does not allow for a straight-forward separation of variables as in Equation~(\ref{case_1_expansion}).  Instead, we begin by making two observations. First, denoting $J = \left[\begin{smallmatrix} 0 & -1 \\ 1 & 0 \end{smallmatrix} \right] \in \SLC$, it is easy to show that 
\begin{align}
    L^{-1} = J^{-1} L^\top J, \label{inv_conj}
\end{align}
where $L^\top \in \textrm{U}$ is the transpose of $L$, belonging to the subgroup $\textrm{U} \in \SLC$ consisting of all \textit{upper-triangular} elements. Critically, $L^\top$ acts on $\RS$ not projectively but as an affine transformation,
\begin{align}
    L^\top z = a^2z + an, \label{LT_act}
\end{align}
equivalent to a planar rotation and dilation, followed by a translation.  Second, for all $z \in \RS$,
\begin{align*}
Jz = J^{-1}z = -z^{-1}
\end{align*}
from which it follows that 
\begin{align}
     J \, \fB_{ms}^t = \ J^{-1} \, \fB_{ms}^t = e^{is\pi} \, \fB_{-m-s}^{-t}. \label{inv_rep}
\end{align}

Combining the observations in Equations~(\ref{inv_conj})~and~(\ref{inv_rep}), we have
\begin{align}
    L \, \fB_{ms}^t & \stackrel{(\ref{inv_conj})}{=}  J^{-1} L^{-\top} J \, \fB_{ms}^t \nonumber \\
    & \stackrel{(\ref{inv_rep})}{=} e^{is\pi} \big[J^{-1} L^{-\top} \fB_{-m-s}^{-t}\big]\label{JT_rep},
\end{align}
Our strategy now becomes clear. Using Equation~(\ref{JT_rep}), we can view the transformation of $\fB_{ms}^t$ by $L$ as the transformation of $\fB_{-m-s}^{-t}$ by $L^{-\top}$, followed by $J$. This allows us to replace the projective action of $\OPS{}$ with the affine action of the upper-triangular subgroup $\textrm{U}$ -- the group of rotations, translations, and dilations -- whose representations are better understood \cite{vilenkin1978special, vilenkin1991representation}. Our goal is now to find an expansion of $L^{-\top} \, \fB_{-m-s}^{-t}$, which we can convert to the desired expansion for $L \, \fB_{ms}^{t} \in \OPS{}$ via the simple action of $J^{-1}$ in Equation~(\ref{inv_rep}). 

We recover an expansion of $L^{-\top} \, \fB_{-m-s}^{-t}$ as follows: First, we apply the Hankel transform in the radial dimension which will allow us to represent $\fB_{-m-s}^{-t}$ in terms of the irreducible unitary representations (\textbf{IURs}) of \textrm{SE}(2) -- the group of planar rotations and translations. From here we can use the regular representation of the group to separate the rotational and translational components of $L^{-\top}$. To handle the remaining dilation, we apply the Mellin transform, which results in the desired expansion. 

To simplify notation, we convert to polar coordinates 
\begin{align*}
    z \mapsto (\abs[]{z}, \textrm{Arg} \, z) \equiv (r, \vartheta).
\end{align*}
In these coordinates $\fB_{ms}^t$ becomes
\begin{align*}
\fB_{ms}^t(z) \mapsto \fB_{ms}^t(r, \vartheta) = r^{is - t} \, e^{im\vartheta}.
\end{align*}
Similarly, we express $a^2$, the rotational and dilational component $L^\top z$,  and $an$, the translational component of $L^\top z$, as
\begin{align*}
a^2 = \alpha e^{ i \varphi} \qquad \textrm{and} \qquad an = \tau e^{i \varkappa}
\end{align*}
for some $\alpha, \tau \in \mathbb{R}_{> 0}$ and $\varphi, \varkappa \in [0, 2\pi)$.

The following calculations were performed with the aid of Mathematica 13.0 \cite{Mathematica}. We begin by expressing $r^{-is + t}$ in terms of its Hankel expansion in the angular frequency $-m$
\begin{align}
    r^{-is + t} = 2^{1-i s + t } \, {\bf R}_{ms}  \int_{0}^{\infty}  \varrho^{is - 1 - t} \,  J_{-m}(\varrho r) \,  d\varrho, \label{r_hankel}
\end{align}
where
\begin{align}
    {\bf R}_{ms} &=   
    \begin{cases}
       \begin{aligned}
        &\frac{ \Gamma \left( 1 - \frac{ m - t + i s}{2}\right)}{\Gamma\left(\frac{-m - t + i s}{2}\right)} &  m \leq 0 \\[1em]
        &(-1)^m \frac{ \Gamma \left( 1 - \frac{-m - t + i s}{2}\right)}{\Gamma\left(\frac{m - t + i s}{2}\right)} & m > 0 
       \end{aligned} 
   \end{cases}. 
\end{align}
and $\Gamma$ and $J_{-m}$ denote the Gamma function and Bessel functions of the first kind, respectively. Substituting the Hankel expansion of $r^{-is + t}$ in Equation~(\ref{r_hankel}) into the polar coordinate expression for $\fB_{-m-s}^{-t}$ gives
\begin{align}
    \fB_{-m-s}^{-t}(r, \vartheta) &= 2^{1-i s + t } \, e^{-i m \vartheta} \,  {\bf R}_{ms}  \int_{0}^{\infty}  \varrho^{is - 1 - t} \,  J_{-m}(\varrho r) \,  d\varrho. \label{fB_hankel}
\end{align}
The matrix elements of the IURs of \textrm{SE}(2) are given by \cite{chirikjian2016harmonic}
\begin{align}
    h_{mn}^{\varrho}(r, \vartheta, \phi) = i^{n - m} \, e^{-in\phi - i(m - n) \vartheta} \, J_{n-m}(\varrho r), \label{se2_iur}
\end{align}
where $\phi$ is the angle of rotation and $r$ and $\vartheta$ are the magnitude and polar angle of the translation, respectively. It follows that Equation~(\ref{fB_hankel}) can be written as 
\begin{align}
    \fB_{-m-s}^{-t}(r, \vartheta) &= 2^{1-i s + t } \, i^{m} \,  {\bf R}_{sm}  \int_{0}^{\infty}  \varrho^{is - 1 - t} \,  h_{m0}^{\varrho}(r, \vartheta, 0) \,  d\varrho.
\end{align} 
From here we can use the regular representation of the group to separate the rotational and translational components of $L^{-\top}$, expanding $L^{-\top} \, \fB_{-m-s}^{-t}$ as
\begin{align}
\begin{aligned}
    &\Big[L^{-\top} \, \fB_{-m-s}^{-t}\Big](r, \vartheta) = \\
    & 2^{1-i s + t } \, i^{m} \,  {\bf R}_{ms} \\
    & \times \sum_{u = -\infty}^{\infty} \int_{0}^{\infty}  \varrho^{is - 1 - t} \, h_{mu}^{\varrho}(\tau, \varkappa, \varphi) \, h_{u0}^{\varrho}(\alpha r, \vartheta, 0) \,  d\varrho.
    \end{aligned} \label{se2_qr}
\end{align}
Expanding the integral in Equation~(\ref{se2_qr}) gives
\begin{align}
& \int_{0}^{\infty} \varrho^{is -1 - t}  h_{mu}^{\varrho}(\tau, \varkappa, \varphi) \, h_{u 0}^{\varrho}(\alpha r, \vartheta, 0) \, d\rho \nonumber \\
& \stackrel{(\ref{se2_iur})}{=} i^{-m} \, e^{-iu(\varphi + \vartheta) -i(m-u)\varkappa}  \nonumber \\
& \qquad \times \int_{0}^{\infty} \varrho^{is -1 - t} \, J_{m - u}(\tau \varrho) \, J_{-u}(\alpha r \varrho) \, d\varrho \nonumber \\
& \stackrel{\phantom{(\ref{se2_iur})}}{=} i^{-m} \, e^{-iu(\varphi + \vartheta) -i(m-u)\varkappa} \, \alpha^{-is + t} \nonumber \\
& \qquad \times \int_{0}^{\infty} \varrho^{is -1 - t} \, J_{m - u}(\alpha^{-1} \tau \varrho) \, J_{-u}(r \varrho) \, d\varrho \label{gotcha} \\
& \stackrel{\phantom{(\ref{se2_iur})}}{=}  2^{is - 1 - t} \, i^{-m} \, e^{-iu(\varphi + \vartheta) -i(m-u)\varkappa} \, \tau^{-is + t} \, M_{smu}^{\, t}(\alpha^2 \tau^{-2} r^2), \label{2nd_hankel}
\end{align}
where the second equality follows from the change of variables $r \mapsto \alpha r$, and the third from evaluation of the integral (the inverse Hankel transform in the $(m-u)-$th frequency). Here,
\begin{align*}
    & M_{msu}^{\, t}(r^2) = \begin{cases}
       \begin{aligned}
           & G_{2,2}^{\,1,1} \, \left( \left. \begin{matrix} {\bf x}_{msu}^{t} \\ {\bf y}_{u} \end{matrix} \, \right| \, r^2 \right) & u \geq m \\[1em]
          & (-1)^{u -m} \, G_{2,2}^{\,1,1} \, \left( \left. \begin{matrix} {\bf x}_{-ms-u}^t \\ {\bf y}_{u} \end{matrix} \, \right| \, r^2 \right) & u < m
       \end{aligned}
    \end{cases}, \\
    {\bf x}_{msu}^t &= \left[\frac{1}{2}( 2 - u - is + m + t), \, \frac{1}{2}(2 + u - is - m + t)\right], \\
    {\bf y}_{u} & =\left[-\frac{1}{2}u, \, \frac{1}{2}u \right],
\end{align*}
with $G_{p,q}^{\,m,n} \, \left( \left. \begin{matrix} {\bf x} \\ {\bf y} \end{matrix} \, \right| \, z \right)$ denoting the Meijer G-function \cite{bateman1953higher}.
Plugging the expression for the integral in Equation~(\ref{2nd_hankel}) into the expression for $L^{-\top} \, \fB_{-m-s}^{-t}$ in Equation~(\ref{se2_qr}) gives
\begin{align}
\begin{aligned}
    &\Big[L^{-\top} \, \fB_{-m-s}^{-t}\Big](r, \vartheta) = \\
    &  {\bf R}_{ms}  \sum_{u = -\infty}^{\infty}  e^{-iu(\varphi + \vartheta) -i(m-u)\varkappa} \, \tau^{-is + t} \, M_{msu}^{\, t}(\alpha^2 \tau^{-2} r^2)
    \end{aligned} \label{penultimate_form}
\end{align}
The above expansion factors out the the rotational and translational components of $L^\top$ as desired, and the final step is to factor out the scale term $\alpha^2 \tau^{-2}$ acting on $r^2$ in the argument of the function $M_{msu}^{\, t}$. 

To do so, we decompose $M_{msu}^{\, t}$ using the Mellin transform. The basis functions of the Mellin transform $r^{\sigma - i\omega}$ are the IURs of the group of dilations acting via multiplication on the positive real line. By decomposing $M_{msu}^{\, t}$ in terms of these bases, we factor out the $\alpha^2 \tau^{-2}$ term in the argument using the regular representation of the group in the same manner as was done in Equation~(\ref{case_1_expansion}). Specifically,
for  $0 \leq t < 1$, and real numbers $\sigma_1, \sigma_2$ satisfying 
\begin{align}
    t < \sigma_1 < 2 \qquad \textrm{and} \qquad t-1 < \sigma_2 < 0 \label{sigma_cond}
\end{align}
$M_{msu}^{t}(r^2)$ can be decomposed as a sum of two Mellin transform expansions 
\begin{align}
    M_{msu}^{t}(r^2) = \frac{1}{2\pi} \sum_{j = 1}^2 \int_{-\infty}^{\infty} \prescript{}{j}{\bf M}_{msu}^{t, \sigma_j}(\omega) \, r^{\sigma_j - i\omega}  d\omega
 \label{mellin_decomp}.
\end{align}
where
\begin{align}
    &\prescript{}{1}{\bf M}_{msu}^{t, \sigma_1}(\omega) = \label{M1_coeff}\\
    &\begin{cases}
       \begin{aligned}
           & \frac{\Gamma\big(  \frac{-\sigma_1 + i\omega - u}{2} \big) \, \Gamma\big(\frac{u + is - m + \sigma_1 - i\omega - t}{2} \big)}{2 \, \Gamma \big(\frac{2 - u + \sigma_1 - i\omega}{2} \big) \, \Gamma \big(\frac{2 + u - is - m - \sigma_1 + i\omega + t}{2}  \big)} & u \geq m, \, u < 0 \\[1em]
           & (-1)^{u} \frac{\Gamma\big(  \frac{u - \sigma_1 + i\omega}{2} \big) \, \Gamma\big(\frac{u + is - m + \sigma_1 - i\omega - t}{2} \big)}{2 \, \Gamma \big(\frac{2 + u + \sigma_1 - i\omega}{2} \big) \, \Gamma \big(\frac{2 + u - is - m - \sigma_1 + i\omega + t}{2}  \big)}   & u \geq m, \, u > 0 \\[1em]
           & (-1)^{u - m} \frac{\Gamma\big(  \frac{-\sigma_1 + i\omega - u}{2} \big) \, \Gamma\big(\frac{-u + is + m + \sigma_1 - i\omega - t}{2} \big)}{2 \, \Gamma \big(\frac{2 - u + \sigma_1 - i\omega}{2} \big) \, \Gamma \big(\frac{2 - u - is + m - \sigma_1 + i\omega + t}{2}  \big)} & u < m, \, u < 0 \\[1em]
          & (-1)^{m} \frac{\Gamma\big(  \frac{u - \sigma_1 + i\omega}{2} \big) \, \Gamma\big(\frac{-u + is + m + \sigma_1 - i\omega - t}{2} \big)}{2 \, \Gamma \big(\frac{2 + u + \sigma_1 - i\omega}{2} \big) \, \Gamma \big(\frac{2 - u - is + m - \sigma_1 + i\omega + t}{2}  \big)} & u < m, \, u > 0 \\[1em]
          & 0 & u = 0, m \neq 0 \\[1em]
          & \frac{\Gamma\big(  \frac{2 - \sigma_1 + i\omega}{2} \big) \, \Gamma\big(\frac{is  + \sigma_1 - i\omega - t}{2} \big)}{2 \, \left(1 - \frac{2 - is + t}{2}\right) \, \Gamma \big(\frac{2 + \sigma_1 - i\omega}{2} \big) \, \Gamma \big(\frac{2 - is  - \sigma_1 + i\omega + t}{2}  \big)} & u = m = 0
       \end{aligned}
    \end{cases} \nonumber
\end{align}
and 
\begin{align}
     &\prescript{}{2}{\bf M}_{msu}^{t, \sigma_2}(\omega) = \label{M2_coeff} \\
     &\begin{cases}
        \begin{aligned}
             & 0 & u \neq 0 \\[1em]
            &\frac{\Gamma\big(  \frac{-\sigma_2 + i\omega}{2} \big) \, \Gamma\big(\frac{is  - m + \sigma_2 - i\omega - t}{2} \big)}{2 \, \Gamma \big(\frac{2 + \sigma_2 - i\omega}{2} \big) \, \Gamma \big(\frac{2 - is  -m - \sigma_2 + i\omega + t}{2}  \big)} & u = 0, \, m < 0 \\[1em] 
             & (-1)^m \frac{\Gamma\big(  \frac{-\sigma_2 + i\omega}{2} \big) \, \Gamma\big(\frac{is  + m + \sigma_2 - i\omega - t}{2} \big)}{2 \, \Gamma \big(\frac{2 + \sigma_2 - i\omega}{2} \big) \, \Gamma \big(\frac{2 - is  + m - \sigma_2 + i\omega + t}{2}  \big)} & u = 0, \, m > 0 \\[1em]
             & \frac{\Gamma\big(  \frac{-\sigma_2 + i\omega}{2} \big) \, \Gamma\big(\frac{2 + is + \sigma_2 - i\omega - t}{2} \big)}{2 \, \left(1 - \frac{2 - is + t}{2}\right) \Gamma \big(\frac{2 + \sigma_2 - i\omega}{2} \big) \, \Gamma \big(\frac{2 - is  - \sigma_2 + i\omega + t}{2}  \big)} & u = m = 0
        \end{aligned}
     \end{cases}. \nonumber
\end{align}
Here, the bounds of $\sigma_1$ and $\sigma_2$ are due to the particular properties of the Mellin transform and ensure that the $M_{msu}^{\, t}$ can be recovered from the coefficients of the forward transform \cite{vilenkin1991representation}.  Then, replacing $M_{msu}^{\, t}(\alpha^2 \tau^{-2} r^2)$ in Equation~(\ref{penultimate_form}) with its Mellin decomposition in Equation~(\ref{mellin_decomp}), rearranging terms, and converting back to complex coordinates $(r, \vartheta) \mapsto (\abs[]{z}, \textrm{Arg} \, z)$ gives 
\begin{align}
    &\Big[L^{-\top} \, \fB_{-m-s}^{-t}\Big](z) = \nonumber \\
    & \begin{aligned}
    &  \frac{{\bf R}_{sm}}{2\pi}   \sum_{u = -\infty}^{\infty} \int_{-\infty}^{\infty} \sum_{j = 1}^2 \fB_{-u-\omega }^{-\sigma_j}(a^2) \, \fB_{u-m \, \omega-s}^{\sigma_j - t}(an) \\
    & \qquad \qquad \qquad  \qquad \times \prescript{}{j}{\bf M}_{msu}^{t, \sigma_j}(\omega) \, \fB_{-u-\omega}^{-\sigma_j}(z) \, d\omega.
\end{aligned}
\label{upper_tri_expansion}
\end{align}
Finally, substituting this expression into Equation~(\ref{JT_rep}) and using Equation~(\ref{inv_rep}) we arrive at the desired expansion of $L \, \fB_{ms}^{t}$:
\begin{align}
    &\Big[L \, \fB_{ms}^{t}\Big](z) = \nonumber \\
    & \begin{aligned}
    &  \frac{{\bf R}_{sm}}{2\pi}   \sum_{u = -\infty}^{\infty} \int_{-\infty}^{\infty} \sum_{j = 1}^2 \fB_{-u-\omega }^{-\sigma_j}(a^2) \, \fB_{u-m \, \omega-s}^{\sigma_j - t}(an) \\
    & \qquad \qquad \qquad  \qquad \times \prescript{}{j}{\bf M}_{msu}^{t, \sigma_j}(\omega) \, \fB_{u  \omega}^{\sigma_j}(z) \, d\omega.
    \end{aligned}
    \label{case_2_expansion}
\end{align}
\noindent \textbf{General Case $(n \in \mathbb{C}$) :}
We combine the expansions of $L \, \fB_{ms}^{t}$ for the cases $n = 0$ in Equation~(\ref{case_1_expansion}) and $n \neq 0$ in Equation~(\ref{case_2_expansion}) into a general form holding for all $L \in \OPS{}$. Specifically, we define the following functions mapping a lower-triangular matrix to a set of filter coefficients,
\begin{align}
    \prescript{u\omega}{1}{\BB}_{ms}^{t \sigma_1}(L) & = (1 - \delta_{\abs[]{n}0}) \frac{{\bf R}_{sm}}{2\pi} \fB_{-u-\omega }^{-\sigma_1}(a^2) \nonumber \\ 
    &\qquad \times \fB_{u-m \, \omega-s}^{\sigma_1 - t}(an) \prescript{}{1}{\bf M}_{msu}^{t, \sigma_j}(\omega), \label{xi_1} \\[1em]
    \prescript{u\omega}{2}{\BB}_{ms}^{t \sigma_2}(L) & =  (1 - \delta_{\abs[]{n}0}) \frac{{\bf R}_{sm}}{2\pi} \fB_{-u-\omega }^{-\sigma_2}(a^2) \nonumber \\
    & \qquad \times \fB_{u-m \, \omega-s}^{\sigma_2 - t}(an) \prescript{}{2}{\bf M}_{msu}^{t, \sigma_2}(\omega), \label{xi_2} \\[1em]
    \prescript{u\omega}{3}{\BB}_{ms}^{t \sigma_3}(L) & =  \delta_{\abs[]{n}0} \delta_{mu} \delta(s - \omega) \, \fB_{-m-s}^{-\sigma_3}(a^2) \label{xi_3},
\end{align}
where $\delta_{xy}$ and $\delta(x)$ denote the Kronecker and Dirac delta functions, respectively, and $\sigma_1, \sigma_2$ satisfy the conditions in Equation~(\ref{sigma_cond}). Given $\fB_{ms}^{t}$ for some $t \in (0, 1)$ and setting  $\sigma_3 = t$, it follows from Equations~(\ref{case_1_expansion})~and~(\ref{case_2_expansion}) that for any $L \in \OPS{}$, $L \, \fB_{ms}^t$ can be expanded as
\begin{align}
   L \, \fB_{ms}^t =  \sum_{u= -\infty}^{\infty} \int_{-\infty}^{\infty}  \sum_{j=1}^3 \prescript{u\omega}{j}{\BB}_{ms}^{t\sigma_j}(L)  \, \fB_{u\omega}^{\sigma_j} \, d\omega \label{basis_expansion}.
\end{align}

\subsection{Transformation of Filters}
Using the expansion of the transformation of basis functions in Equation~(\ref{basis_expansion}), it is straight-forward to recover the expansion of the transformation of filters $f$ of the form
\begin{align*}
    f = \sum_{m=-M}^M\sum_{s=-N}^N  b_{ms} \, \fB_{ms}^t,
\end{align*}
by elements of $\OPS{}$. Namely, 
\begin{align}
\begin{aligned}
    L \, f & \stackrel{\phantom{(\ref{basis_expansion})}}{=} \sum_{m=-M}^M\sum_{s=-N}^N  b_{ms} \, L \, \fB_{ms}^t \\
    & \stackrel{(\ref{basis_expansion})}{=} \sum_{m=-M}^M\sum_{s=-N}^N b_{ms} \sum_{u= -\infty}^{\infty} \int_{-\infty}^{\infty}  \sum_{j=1}^3 \prescript{u\omega}{j}{\BB}_{ms}^{t\sigma_j}(L)  \, \fB_{u\omega}^{\sigma_j} \, d\omega \\
    & \stackrel{\phantom{(\ref{basis_expansion})}}{=} \sum_{u= -\infty}^{\infty} \int_{-\infty}^{\infty}  \sum_{j=1}^3 \underbrace{\left[\sum_{m=-M}^M\sum_{s=-N}^N b_{ms}  \prescript{u\omega}{j}{\BB}_{ms}^{t\sigma_j}(L) \right]}_{\prescript{}{j}{\LB}_{u\omega}^{t \sigma_j}(L, {\bf b})}  \, \fB_{u\omega}^{\sigma_j} \, d\omega, 
    \end{aligned}\label{filter_expansion}
\end{align}
where  ${\mathbf b}$ is the $(2M+1)\times(2N+1)$-dimensional vector of coefficients of $f$ and $\prescript{}{j}{\LB}_{u\omega}^{t \sigma_j}$ maps a lower-triangular element and a set of filter coefficients to the coefficient of $\fB_{u\omega}^{\sigma_j}$ in the expansion. 

\subsection{Implementation} \label{a:quad}

As discussed in \S\ref{s:discretization} we approximate the expansion of $L \, f$ by truncating the summation over $u$, and replacing the integration over the real line with a discrete approximation using quadrature. The summation is a consequence of the addition theorem for Bessel functions \cite{watson1995treatise}. Here it arises in the regular representation of $\textrm{SE}$(2) used in Equation~(\ref{se2_qr}) to factor the rotational and translational components of the transformations from the argument of the basis functions.  Fortunately, it converges rapidly for low angular basis frequencies $m$ and we typically find truncation at $M + 1$ terms to be sufficient.

Approximating the integral in the expansion is less straight-forward. For example, the reader may have noticed that the second to last equality -- Equation~(\ref{gotcha}) -- in the expansion of the integral in Equation~(\ref{se2_qr}) provides a seemingly suitable separation of variables for our purposes, raising the question of why we expend the additional effort dealing with the Mellin transform. The problem with the expansion offered by Equation~(\ref{gotcha}) is that the product of Bessel functions in the integrand is highly-oscillatory, and decays either very rapidly or very slowly depending on the values of $\alpha, \tau$ and $r$ making a low-order numerical integration scheme impossible.

However, it turns out that first recollecting the separated terms by evaluating the integral (the inverse Hankel transform) -- Equation~(\ref{2nd_hankel}) -- and then expanding the solution again using the Mellin transform -- Equation~(\ref{mellin_decomp}) -- gives us something we can handle numerically. (Equivalently, we could have first expanded $J_{-u}(r\rho)$ in Equation~(\ref{gotcha}) using the Mellin transform, then applied the inverse Hankel transform to arrive at a similar expression). Despite being aesthetically-challenged, the Mellin transform coefficients $\prescript{}{j}{\bf M}_{msu}^{t, \sigma_j}$  in Equations~(\ref{M1_coeff} - \ref{M2_coeff}) have several nice properties which make possible a low-order quadrature approximation of the integral: They decay rapidly, are relatively smooth, and retain these properties even with increasing  values of $\abs[]{m}, \abs[]{s}$ and $\abs[]{u}$.

\subsubsection*{Quadrature}

For a given choice of $t \in (0, 1)$ determining the localization of the filters, the smoothness and decay of the Mellin coefficients can further be controlled by the choices of $\sigma_1$ and $\sigma_2$ satisfying Equation~(\ref{sigma_cond}) such that the discretization of the Mellin expansion of $M_{msu}^t(r^2)$ in Equation~(\ref{mellin_decomp}) with a fixed number of samples $\{\omega_q\}_{1 \leq q \leq Q}$ and trapezoidal quadrature weights $\{w_q\}_{1 \leq q \leq Q}$,
\begin{align}
\begin{aligned}
&M_{msu}^t(r^2; \sigma_1, \sigma_2, \omega_1, \ldots, \omega_Q) \equiv \\
& \frac{1}{2\pi} \sum_{j = 1}^2 \sum_{q=1}^Q  w_q  \prescript{}{j}{\bf M}_{msu}^{t, \sigma_j}(\omega_q) \, r^{\sigma_j - i \omega_q},
\end{aligned} \label{disc_decomp}
\end{align}
provides the best approximation.

In practice, we formulate the selection of $\sigma_1$ and $\sigma_2$ as an optimization problem: Given a fixed number of quadrature samples $Q$, we treat $\sigma_1, \, \sigma_2$ and the quadrature sample points $\{\omega_q\}_{1 \leq q \leq Q}$ as optimization variables, and seek to minimize the reconstruction error between $M_{msu}^t(r^2)$ and the discretization of its Mellin expansion in Equation~(\ref{disc_decomp}):
\begin{align}
    \begin{aligned}
        & {\mathcal E}_{msu}^t [\sigma_1, \sigma_2, \omega_1, \ldots, \omega_Q] =  \\ 
    &\int_{0}^{\infty} \abs[\Big]{ M_{msu}^t(r^2) - M_{msu}^t(r^2; \sigma_1, \sigma_2, \omega_1, \ldots, \omega_Q) }^2 \, dr
    \end{aligned} \label{recon_error}
\end{align} However, instead of computing a distinct quadrature scheme to approximate each function $M_{msu}^t(r^2)$, for each value of $u$, we optimize for a quadrature scheme that best reconstructs $M_{msu}^t(r^2)$ for all tuples $(m, s, u)$ so that the  bases $\fB_{u \omega}^{\sigma_j}$ remain independent of $m$ and $s$ in the transformation of an $M, N$ band-limited filter as in Equation~(\ref{approx}). That is, for each $-M' \leq u \leq M'$, we choose $\sigma_1^u, \sigma_2^u$ and $\{\omega_q^u\}_{1 \leq q \leq Q}$ to be the minimizers of 
\begin{align}
\begin{aligned}
    {\mathcal E}_u^t[\sigma_1^u, \sigma_2^u, \omega_1^u, \ldots, \omega_Q^u] = & \sum_{\substack{-M \leq m \leq M \\ -N \leq s \leq N}} {\mathcal E}_{msu}^t [\sigma_1^u, \sigma_2^u, \omega_1^u, \ldots. \omega_Q^u],
\end{aligned} \label{quad_energy}
\end{align}
and are solved for via gradient descent in a pre-processing step. Once the sample points have been recovered, the values of Mellin transform coefficients $\prescript{}{j}{\bf M}_{msu}^{t, \sigma_j}(\omega_q^u)$ in Equations~(\ref{xi_1}-\ref{xi_2}) can be computed in the same pre-processing regime, avoiding the evaluation of the Gamma functions at run-time.

%% file: supp/supp_overhead.tex
\subsection{Memory Footprint}
The most computationally expensive step of our \mob{}-equivariant CNNs is the innermost sum of the reduction inside the convolutional layers in Equation~(\ref{conv_reduction}),
\begin{align*}
 \feat{}_{c'}' = \sum\limits_{\substack{-M'\leq m\leq M'\\1\leq q \leq Q\\1\leq j\leq 3}}
 \left( \, \sum_{c=1}^{C} \,  \h_{\feat{}_c} \, w_q \, \prescript{}{j}{\LB}_{ms_q}^{ t\sigma_j}\big(\T_{\feat{}_c},{\mathbf b}^{cc'}\big) \, \, {*}_{e} \, \, \fB_{ms_q}^{\sigma_j} \, \right).
\end{align*}
In practice we only sum over the first two indices of $j$, replacing $an$ with a small constant factor $\varepsilon = 0.05$ whenever it vanishes. In addition, it follows from Equation~(\ref{M2_coeff}) that for $j = 2$, $\prescript{}{2}{\LB}_{ms_q}^{ t\sigma_2}\big(\T_{\feat{}_c},{\mathbf b}^{cc'}\big)$ vanishes whenever $m \neq 0$. Then, using the expansion
\begin{align}
    \prescript{}{j}{\LB}_{ms_q}^{ t\sigma_j}\big(\T_{\feat{}_c},{\mathbf b}^{cc'}\big) \stackrel{(\ref{filter_expansion})}{=} \sum_{k=-M}^M\sum_{l=-N}^N b_{kl}^{cc'}  \prescript{ms_q}{j}{\BB}_{kl}^{t\sigma_j}(\T_{\feat{}_c}),
\end{align}
the inner sum in the reduction is expressed as broadcasted matrix multiplication,
 $\boldsymbol{b} \, \boldsymbol{\Xi}$,
where $\boldsymbol{b}$ is the $C' \times C\cdot (2M + 1) \cdot (2N + 1)$ dimensional $64$-bit complex tensor of learnable filter parameters $b_{kl}^{cc'}$ and $\boldsymbol{\Xi}$ is the $C\cdot (2M + 1) \cdot (2N + 1) \times (2M' + 2) \times Q \times 2B \times 2B$ $64$-bit complex tensor corresponding to the values of $\prescript{ms_q}{j}{\BB}_{kl}^{t\sigma_j}(\T_{\feat{}_c})$. The   first $(2M' + 1)$ indices in the second dimension of $\boldsymbol{\Xi}$ correspond to $j = 1$ with $-M' \leq m \leq M$ and the last index corresponds to $j = 2$ with $m = 0$.

 Typical layers use $M=N=1$ band-limited filters, with $M' = 2$ and $Q = 30$ quadrature points. Here, $\boldsymbol{\Xi}$ is a $CB^2 51840$-byte tensor which must be stored on device memory. As such, \mob{} convolution layers have a large memory footprint at high resolutions $B \geq 64$ -- each input channel requires approximately $0.2$ GB or more of device memory per convolution. We find that memory overhead and inference time can both be reduced by fusing both the creation of $\boldsymbol{\Xi}$ and matrix multiplication by $\boldsymbol{b}$ into a single operation (\textit{e.g.} via TorchScript). That said, all of our experiments are performed using an NVIDIA RTX A6000 GPU and the \mob{} convolution U-Net used in the segmentation task in \S\ref{segmentation_section} consumes approximately $80\%$ of the $48$ GB of device memory. 

\subsection{Run Time}
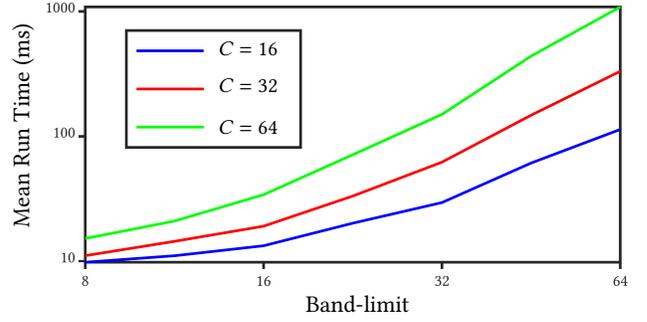
\begin{figure}[t]
\begin{picture}(0.925\linewidth,0.505\columnwidth)
\put(5,5){\includegraphics[width=0.95\linewidth]{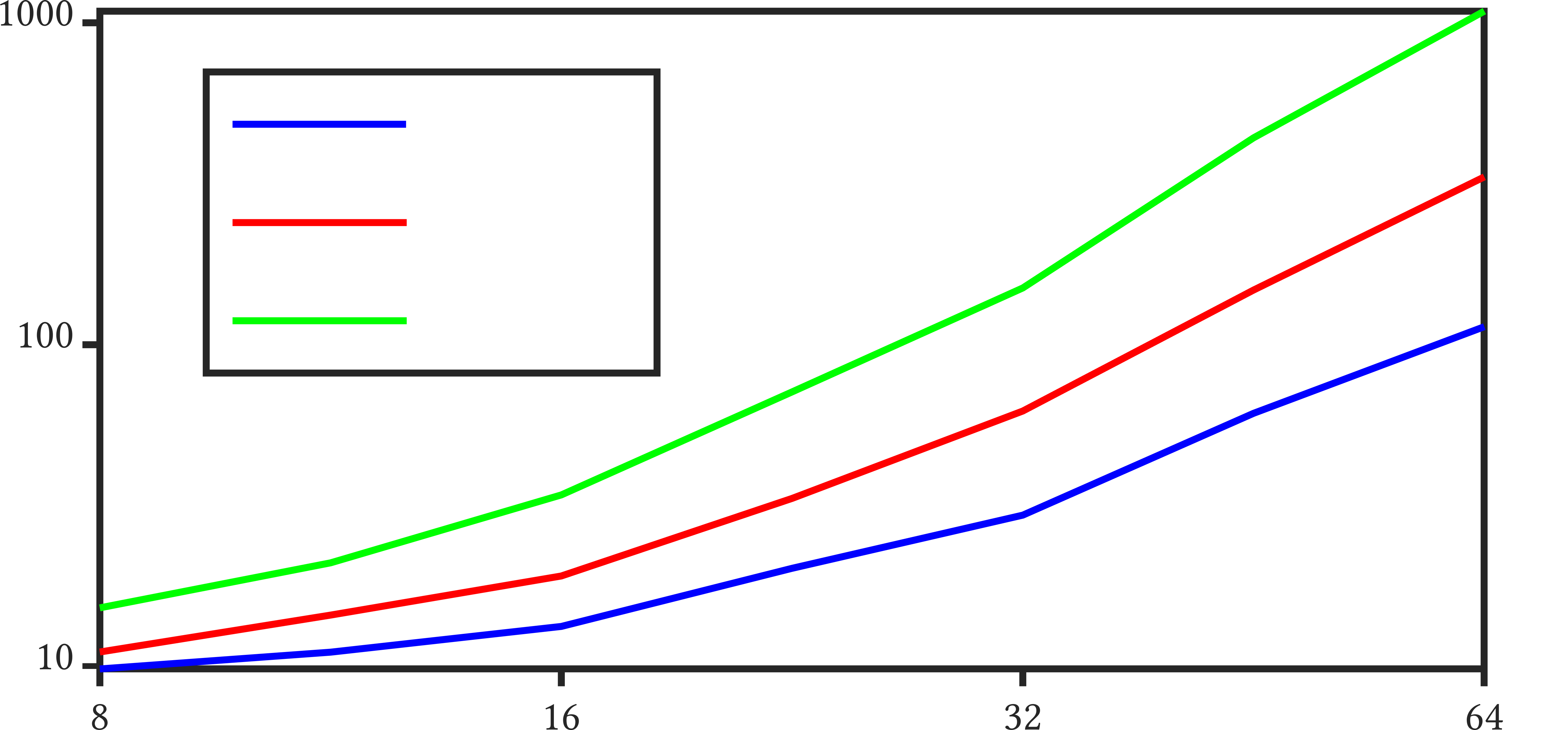}}
\put(103, -5){Band-limit}
\put(-8, 27){\rotatebox{90}{Mean Run Time (ms)}}
\put(68, 92){ {\small $C = 16$}}
\put(68, 78){ {\small $C = 32$}}
\put(68, 62.5){ {\small $C = 64$}}

\end{picture}
    \caption{Mean run time of a $(M=1, N=1, Q=2, P=30)$ \mob{} convolution block with $C = 16, 32, 64$ channels over $100$ initializations as a function of band-limit. \label{time_plot}
 }
\end{figure}

Figure~\ref{time_plot} shows the mean run time of a single \mob{} convolution block with increasing numbers of channels as a function of band-limit. Despite having a large memory footprint, \mob{} convolutions are relatively fast -- only the forward passes of high band-limit, large-width modules exceed $100$ ms. Training the networks in the classification (\S\ref{sec:classification}) and segmentation (\S\ref{segmentation_section}) tasks takes approximately $30$ minutes and four hours, respectively. 

%% file: supp/supp_seg.tex
\begin{figure*}[t]
    \includegraphics[width=\textwidth]{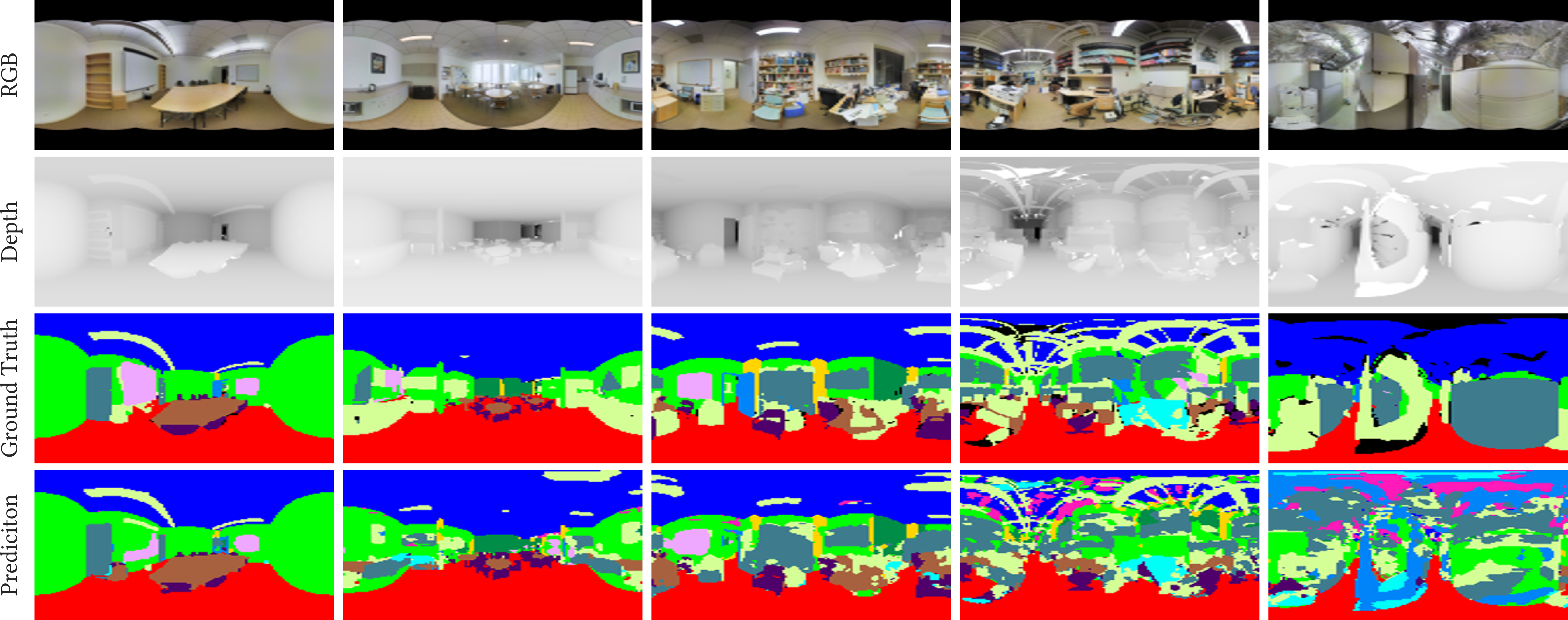}
        \caption{Visualization of semantic segmentation results on images in the test set. \label{seg_fig}}
\end{figure*}
We provide a visualization of our model's predicted segmentation labels relative to the ground truth in Figure~\ref{seg_fig}. Our model performs well in simpler scenes with complete depth information (left), but becomes less effective as label granularity increases (moving to the right). Our model fails (far right) in the presence of large, specular surfaces.